\documentclass{article}
\usepackage[margin=1.25in]{geometry}
\usepackage{amsmath,amsfonts,amsfonts,amsthm,amssymb}

\usepackage{times}
\usepackage{graphicx} % more modern
\usepackage{caption}
\usepackage{subcaption}
\usepackage{footnote}
\makesavenoteenv{tabular}
\makesavenoteenv{table}
\usepackage{authblk}
\usepackage{natbib}
\usepackage{amsmath}
\usepackage{amssymb}
\usepackage{tikz}
\usepackage{xcolor}
\usetikzlibrary{arrows}
\usepackage{hyperref}

\usepackage{algorithm}
\usepackage{algorithmic}
\usepackage{hyperref}
\usepackage{bm}
\usepackage{footnote}

\usepackage{varwidth}
\newcommand{\cov}{\mathrm{cov}}

\newcommand{\argmax}{\mathrm{argmax}}

\newcommand{\tr}{\mathrm{tr}}

\newcommand{\goodind}{\mathcal{G}}
\newcommand{\badind}{\mathcal{B}}
\newcommand{\func}{g}
\newcommand{\meandiff}{\Delta}
\newcommand{\covdiff}{\mathcal{E}}
\newcommand{\lip}{L}

\newcommand{\expect}{\mathbb{E}}
\newcommand{\holder}{H\"{o}lder }
\newcommand{\ind}{1}

\newcommand{\set}[1]{\mathcal{#1}}
\newcommand{\mat}[1]{{#1}}
\newcommand{\vect}[1]{\mathrm{vec}\left({#1}\right)}
\newcommand{\norm}[1]{\left\|#1\right\|}
\newcommand{\abs}[1]{\left|#1\right|}
\newcommand{\prob}{\mathbb{P}}

\newcommand{\tops}{P_{2s}}

\newtheorem{thm}{Theorem}[section]
\newtheorem{lem}{Lemma}[section]
\newtheorem{cor}{Corollary}[section]

\newtheorem{defn}{Definition}[section]

\newcommand{\samplesize}{n}
\newcommand{\fronorm}[1]{\ensuremath{\matsnorm{#1}{\footnotesize{\mbox{F}}}}}
\newcommand{\opnorm}[1]{\ensuremath{\matsnorm{#1}{\footnotesize{\mbox{op}}}}}

\newcommand{\matsnorm}[2]{\left|\!\left|\!\left|#1\right|\!\right|\!\right|_{{#2}}}

\newcommand{\inprod}[2]{\ensuremath{\langle #1 , \, #2 \rangle}}
\newcommand{\sparseopnorm}[1]{\ensuremath{\matsnorm{#1}{\footnotesize{\mbox{s,op}}}}}

\usepackage{times}

\title{Computationally Efficient Robust Estimation of Sparse Functionals}
\author[1]{Simon S. Du\thanks{ssdu@cs.cmu.edu}}
\author[12]{Sivaraman Balakrishnan\thanks{siva@stat.cmu.edu}}
\author[1]{Aarti Singh\thanks{aarti@cs.cmu.edu}}
\affil[1]{Machine Learning Department, Carnegie Mellon University, USA}
\affil[2]{Statistics Department, Carnegie Mellon University, USA}

\begin{document}

\maketitle
\begin{abstract}
\label{sec:abs}
Many conventional statistical procedures
are extremely sensitive to seemingly
minor deviations from modeling assumptions.
This problem is exacerbated in modern 
high-dimensional settings, where the problem dimension can
grow with and possibly exceed the sample size. 
We consider the problem of \emph{robust
estimation} of sparse functionals, and 
provide a computationally
and statistically efficient algorithm in the high-dimensional setting. 
Our theory identifies
a unified set of deterministic conditions under which our algorithm guarantees
accurate recovery. 
By further establishing that these deterministic 
conditions hold with high-probability for
a wide range of statistical models, our theory applies to many problems 
of considerable interest including sparse mean and covariance 
estimation; sparse linear regression; and sparse generalized linear models. 
%an $s$-sparse functional of a target distribution in Huber's $\epsilon$-contamination model.
%In this work, we obtain the first computationally efficient algorithm for estimating high dimensional sparse functionals with error guarantees independent of both sparsity $s$ and ambient dimension $d$. 
%Our estimator works for a wide range of statistical estimation problems satisfying some mild concentration properties.
%In particular, we apply our estimator to sparse Gaussian mean, sparse Gaussian covariance and sparse linear regression problem and achieve near-optimal sample complexity.
\end{abstract}

%\begin{keywords}
%Robustness, Sparsity, linear regression, covariance estimation, generalized linear models, logistic regression 
%\end{keywords}

\section{Introduction}
\label{sec:intro}
Complex high-dimensional datasets pose a variety of computational and statistical challenges.
In attempts to address these challenges, the past decade has witnessed a significant amount of research on sparsity constraints in statistical
models. Sparsity constraints have practical and theoretical benefits: often they lead to more interpretable models, that can be estimated efficiently even in the high-dimensional regime where
the sample size $n$ can be dwarfed by the model dimension $d$. In addition to being convenient
from a methodological and theoretical standpoint, sparse models have also had enormous practical impact, for instance in computational biology, neuroscience and applied machine learning. 

On the other hand, much of the theoretical literature on sparse estimation has focused on providing guarantees under strong, often impractical, generative assumptions. This in turn motivates the study of the robustness of these statistical estimators, and the design of new robust estimators. Classically,
the sensitivity of conventional statistical procedures to apparently small deviations from the assumed statistical model, was noted by~\cite{tukey1975mathematics} who observed that estimators like the empirical mean can be sensitive to even a single gross outlier. The formal study,
of robust estimation, was initiated by \citet{huber1964robust,huber1965robust}, who initiated
the study of estimation procedures under the \emph{$\epsilon$-contamination model}, where samples are obtained from a mixture model of the form:
\begin{align}
P_{\epsilon} = (1 - \epsilon) P + \epsilon Q \label{eqn:model}, 
\end{align}
where $P$ is the uncontaminated target distribution, $Q$ is an arbitrary outlier distribution and $\epsilon$ is the expected fraction of outliers. Subsequent work in the literature
on robust statistics, focused on the design of robust estimators and the study of their 
statistical properties (see, for instance, the works
of \citet{huber2011robust,hampel2011robust}). Recent research~\citep{chen2015robust,chen2016general}
has focussed on providing a complementary
minimax perspective by characterizing both minimax upper and lower bounds on the
performance of estimators in a variety of settings. Notably, the minimax estimation 
rates in these settings typically
have two aspects: 
(1) the dependence on the contamination parameter $\epsilon$, which we refer to as the \emph{contamination dependence},
and (2) the \emph{statistical rate} (typically, a function of the sample size $n$ and the dimensionality $d$). 

The major drawback of many of 
these classical robust estimators is that they are either heuristic
in nature (for instance, methods based on Winsorization ~\citep{hastings1947low}) 
and are generally not optimal in the minimax sense, 
or are computationally intractable (for instance, methods based on Tukey's depth~\citep{tukey1975mathematics} or on $\ell_1$ tournaments~\citep{yatracos1985rates}).

Considering the low-dimensional setting where $d \ll n$, recent works~\citep{diakonikolas2016robust,lai2016agnostic,charikar2016learning} provide some of the first computationally tractable, provably robust estimators with near-optimal contamination dependence in a variety of settings. Concretely, the paper of \citet{lai2016agnostic} considers robust mean and covariance estimation for distributions with appropriately controlled moments, while the work of \citet{diakonikolas2016robust}, focuses on robust mean and covariance estimation for Gaussians and extends these results to various other models including the mixture of Gaussians. 
After the initial submission of this manuscript we became aware of independent and concurrent work by \citet{li2017robust}, addressing similar high-dimensional concerns. In more details, \citet{li2017robust} uses a similar algorithm for mean estimation, and also considers the problem of robust sparse PCA.

% We build extensively on the results of \citet{diakonikolas2016robust}, in order to provide novel 
Although the focus of our paper is on the multivariate setting, we note that several recent papers
have provided robust estimation guarantees for univariate distributions \citep{acharya2017sample,chan2013learning,chan2014near,daskalakis2012learning,diakonikolas2016efficient}.

We make several contributions to this line of research. 
In more details, we 
focus on the \emph{sparse high-dimensional setting} where the dimensionality $d$ is potentially much larger than the sample-size $n$ but the unknown target parameter is $s$ sparse (and $s \ll n$).
Building on the work of \citet{diakonikolas2016robust}, our first main contribution is to provide a unified
framework for the estimation of sparse functionals. We identify a set of core deterministic conditions,
under which we can guarantee accurate recovery of a statistical functional in polynomial-time. In contrast to prior work, this framework unifies, for instance, the robust estimation of the mean vector 
and of the covariance matrix of a high-dimensional distribution.
Our second main contribution, establishes that these deterministic conditions hold with high-probability in many statistical models, even in the high-dimensional setting where $n \ll d$, under appropriate sparsity assumptions. As a consequence, we obtain the first robust estimators in a variety of high-dimensional problems of practical interest including sparse mean and covariance 
estimation; sparse linear regression; and sparse generalized linear models. Finally, from a technical standpoint, as will be discussed at more length in the sequel 
we introduce a variety of new techniques involving 
the careful analysis of convex relaxations and delicate truncation arguments that we anticipate will be useful in other related problems.

The remainder of the paper is organized as follows. In Section~\ref{sec:background}, we provide some background on robust estimation and formally introduce the examples we consider throughout this paper. Section~\ref{sec:main} is devoted to our main results and their consequences. Section~\ref{sec:techniques} includes a description of our main algorithm, and includes a sketch of its analysis with more technical details deferred to the Appendices. 
We conclude with a brief discussion of avenues for future work.

\section{Background and Problem Setup}
\label{sec:background}
In this section we provide some background on robust estimation, before providing a precise definition
of the statistical models we consider.
\subsection{Robust estimation}
\label{sec:notations}
In the robust estimation framework we suppose that we obtain samples
$\{x_1,\ldots,x_{\samplesize}\}$ where each sample $x_i$ is distributed according to the
mixture model
$P_{\epsilon}$ in Eqn.~\eqref{eqn:model}:
\begin{align*}
P_{\epsilon} = (1 - \epsilon) P + \epsilon Q.
\end{align*}
In this model, the distribution $Q$ is allowed
to be completely arbitrary and represents the distribution of ``outliers''. As no restriction is placed 
on $Q$ the sampled outliers can be gross outliers, i.e. starkly distinct from the bulk of the samples from $P$ or more subtle outliers which are more difficult to tease apart from samples from the target distribution.  An alternative viewpoint arises from the observation that the set of possible distributions $P_{\epsilon}$ is equivalent to the $\ell_1$ ball around $P$ of radius $\epsilon$. Indeed, we can alternatively view desirable estimators in this model as those that are robust to model-misspecification (in the $\ell_1$ or total variation metric).
We note in passing, that the work of \citet{diakonikolas2016robust}, focuses on a stronger adversarial model where an adversary is allowed to replace samples from the target distribution, in addition to adding outliers.

Our focus, will be on finite-dimensional 
functionals of the target distribution $P$. Formally, for 
a given function $g: \mathbb{R}^{\widetilde{d}} \mapsto \mathbb{R}^d$,
we define the corresponding functional $\theta_g$ as a mapping, 
$\theta_g: P \mapsto \mathbb{R}^d$, where:
\begin{align*}
\theta_g(P) = \mathbb{E}_{x \sim P} [g(x)].
\end{align*}
Motivated by similar considerations in high-dimensional statistics,
our sparsity assumption will be that 
$\|\theta_g(P)\|_0 \leq s$.
We will further denote the covariance 
%function by the map $\text{cov}: \mathbb{R}^{\widetilde{d}} \mapsto
%\mathbb{R}^{\widetilde{d} \times \widetilde{d}}$ defined 
as,
\begin{align}
\text{cov}(\theta_g(P)) = \mathbb{E}_{x \sim P}\left[ (g(x) - \theta_g(P))(g(x) - \theta_g(P))^T\right] \label{eqn:cov}
\end{align}
Our algorithm will be based on trying to appropriately weight samples in order to match 
second order information and in order to accomplish this we will rely on the existence of an algebraic form for the covariance. In particular, we will suppose that there exists a multivariate function $F: \mathbb{R}^d \mapsto
\mathbb{R}^{d \times d},$ such that $F(\theta_g(P)) = \text{cov}(\theta_g(P)).$
An important restriction on this algebraic form, that will enable accurate estimation, is that it be sufficiently regular. 
Concretely, first, we assume \begin{align}
\lip_{\cov} = \max_{\norm{v}_2, \norm{v}_0\le s}\abs{v^\top \cov\left(\theta_\func\right)v} \label{eqn:cov_sparse_bound}
\end{align} for some constant $L_\cov$.
Second, we require that for any two vectors
$\theta_1, \theta_2 \in \mathbb{R}^d$, there exist a constant $\lip_F$ and a universal constant $C$
such that
\begin{align}
\opnorm{F(\theta_1) - F(\theta_2)} \leq \lip_F \|\theta_1 - \theta_2 \|_2 + C \|\theta_1 - \theta_2 \|_2^2. \label{eqn:F_lip}
\end{align}
Our bounds depend explicitly on $\lip_F$ and $\lip_\cov$.
In the next subsection, we consider a variety of examples and describe the appropriate
functionals of interest and their corresponding covariance.

\subsection{Illustrative examples}
Our general results apply to a variety of statistical models and in this section we describe a few concrete examples of interest. 

\noindent {\bf Sparse Gaussian Mean Estimation: } In this setting, we observe samples 
\begin{align}
\label{eqn::mean}
\{x_1,\ldots,x_n\} \sim (1 - \epsilon) N(\mu, I) + \epsilon Q,
\end{align} 
where each $x_i \in \mathbb{R}^d$ and for an arbitrary $Q$ \footnote{We address the unknown covariance case in the sequel.}.
The goal in this setting is to estimate $\mu$ in the $\ell_2$ norm in the high-dimensional setting, under the assumption of sparsity, i.e. that $\|\mu\|_0 \leq s$. Using the notation introduced earlier, the function $g$ is simply the identity, i.e. $g(x) = x$. In this setting, as will be clearer in the sequel 
it will be convenient to simply define the covariance map via its known value at the true parameter, i.e. we take $F(\theta) = I.$

\noindent {\bf Sparse Gaussian Covariance Estimation: } In this case, we observe samples
\begin{align}
\label{eqn::covmodel}
\{x_1,\ldots, x_n\} \sim (1-\epsilon) N(0,\Sigma) + \epsilon Q,
\end{align}
where each $x_i \in \mathbb{R}^d$ and where the covariance matrix can be written as $\Sigma = I + S,$ 
where $\|S\|_0 \leq s$. The goal in this problem is to estimate the sparse matrix $S$. This problem is closely related to the problem of Gaussian graphical modeling, where a typical assumption is sparsity of the precision matrix $\Sigma^{-1}$. Zeros in the precision matrix correspond to conditional independencies, on the other hand zeros in $\Sigma$ signal marginal independencies, and this can be used to construct a graphical display of the relationship between the features~\citep{bien2011sparse}.

In this problem, denoting by $\text{vec}(M)$ the vectorization of the matrix $M$,
and by $\text{diag}(M)$ its diagonal entries, 
we consider the function 
$g(x) = \text{vec}(xx^T - \text{diag}(xx^T))$.
Further, using $\otimes$ to denote the Kronecker product, we have that:
\begin{align*}
F(\text{vec}(S)) = \text{vec}(S) \text{vec}(S)^T + S \otimes S.
\end{align*}

Finally, we note that via a simple reduction scheme (described in detail
in \citet{diakonikolas2016robust}) we can combine the above two settings in order to jointly
estimate an unknown mean, and an unknown covariance robustly in a high-dimensional setting 
provided both are sparse.
Concretely, we can take the difference of two batches of 
samples, in order to obtain samples with zero mean, twice the covariance, and with a higher contamination rate (roughly doubling $\epsilon$). 
We can then estimate, the covariance matrix to a sufficient accuracy, and use the estimate to whiten
a fresh batch of samples. We can finally estimate the mean on this whitened batch of samples.

\noindent {\bf Linear Regression: } Linear regression is a canonical problem in statistics. In the uncontaminated setting we  
observe paired samples $\{(y_1,x_1),\ldots,(y_n,x_n)\}$ which are related via the linear model,
\begin{align}
\label{eqn::regmodel}
y_i = \inprod{x_i}{\beta} + \epsilon_i,
\end{align}
where $x_i, \beta \in \mathbb{R}^d$ and $\epsilon_i \in \mathbb{R}$ is some type of observation noise. 
In this paper, we assume that $x_i \sim N(0,I)$ and $\epsilon_i \sim N(0,1)$, and our goal is to estimate the unknown $\beta$ in a high-dimensional setting under the assumption that $\|\beta\|_0 \leq s$. In this problem, we take $g((y,x)) = yx$, by making the observation
that the functional of interest 
$\beta = \mathbb{E}[yx]$. Further, we can calculate the algebraic form for the covariance as:
\begin{align*}
F(\beta) = (\|\beta\|_2^2 + 1)I + \beta \beta^T.
\end{align*}

\noindent {\bf Generalized Linear Models (GLMs): } We consider two distinct forms for GLMs in our work. The first form is a non-linear regression model where the uncontaminated distribution $P$
corresponds to pairs $\{(y_1,x_1),\ldots,(y_n,x_n)\}$ which are related as,
\begin{align}
\label{eqn::glmmodel}
y_i = u( \inprod{x_i}{\beta} ) + \epsilon_i
\end{align}
where $u$ is a known non-linear function,  $x_i, \beta \in \mathbb{R}^d$ and $\epsilon_i \in \mathbb{R}$. 
As before we assume that, $x_i \sim N(0,I)$, $\epsilon_i \sim N(0,1)$, and further that there exist
constants $C_1$ and $C_2$ such that, $u(0) \leq C_1$ and $u$ is $C_2$-Lipschitz, i.e. for any pair $x, y \in \mathbb{R}$ we have that,
\begin{align*}
|u(x) - u(y)| \leq C_2 |x - y|.
\end{align*}
The goal is to estimate the unknown, sparse $\beta$. In this case, we choose 
$g((y,x)) = \frac{xy}{\mathbb{E} [\nabla_{x'} u(x') ] }$ 
where $x' = \inprod{x}{\beta}$.
As a consequence of Stein's identity
we have that $\mathbb{E}[g((y,x))] = \beta$. 
%In the sequel, we provide a direct robust estimate
%of the scalar normalizer $\mathbb{E}[\nabla_x g(x)]$. 
%1+E g2(x?) ,?3 =E 2g(x?)?2x?g(x?)+?x?g(x?)2
Once again, by Stein's identity (see Appendix~\ref{sec:sparse_glm}) we obtain the algebraic form of the covariance:
\begin{align*}
F(\beta) = \left( \frac{1 + \mathbb{E} [u^2(x')] }{\left(\mathbb{E} [\nabla_{x'} u(x') ] \right)^2 }  \right) I + \left(\frac{\mathbb{E} [2 u(x') \nabla^2_x u(x') + (\nabla u(x'))^2 ] }{ \left(\mathbb{E} [\nabla_x' u(x') ] \right)^2 } \right) \beta \beta^T.
\end{align*} where $x' = \inprod{x}{\beta}$.
Observe that $F(\beta)$ has the form $\kappa_1 \mat{I} + \kappa_2 \beta \beta^\top$ where $\kappa_1$ and $\kappa_2$ are scalars.
Further notice that $x' \sim N\left(0,\norm{\beta}_2^2\right)$, so these quantities can be estimated easily using just $\left\{y_1,\ldots,y_n\right\}$ with a one-dimensional robust method like the median estimator.
Therefore, from now on, we will assume these quantities are known constants.
%Although this is not an explicit algebraic form, 
%in the sequel, we provide a robust estimate of the scalar scaling factors which suffices for our purposes.

\noindent {\bf Logistic-type Models: } Finally, our theory also applies to GLMs of the logistic regression form. In the uncontaminated setting, we observe pairs  $\{(y_1,x_1),\ldots,(y_n,x_n)\}$,
where $y_i \in \{0,1\}$ where, 
\begin{align*}
\label{eqn::logisticmodel}
\mathbb{P}(y_i = 1 | x_i) = u(\inprod{x_i}{\beta}),
\end{align*}
and the assumptions on $x$ and $u$ are as before. In this case, the function $g$ is identical to the
previous case, and its corresponding covariance is given as (see Appendix~\ref{sec:sparse_classification}):
\begin{align*}
F(\beta) = \left( \frac{ \mathbb{E} [u(x')] }{\left(\mathbb{E} [\nabla_{x'} u(x') ] \right)^2 }  \right) I + \left(\frac{\mathbb{E} [ \nabla^2_{x'} u(x') - (\nabla u(x'))^2 ] }{ \left(\mathbb{E} [\nabla_{x'} u(x') ] \right)^2 } \right) \beta \beta^T.
\end{align*} where $x' = \inprod{x}{\beta}$.

With these preliminaries in place, we devote our next section to a description of our main results concerning the robust high-dimensional estimation of these statistical models. 

\section{Main Results}
\label{sec:main}

We begin this section by identifying a set of deterministic conditions under which we can design a polynomial time algorithm that is provably robust. 

\subsection{Main deterministic result}
Our main result is based on obtaining adequate control on the 
deviations of a collection of 
moment-type quantities which are obtained by 
appropriately weighted 
versions of the observed sample from their population counterparts. Concretely, we observe samples
\begin{align*}
\{x_1,\ldots,x_n\} \sim (1 - \epsilon) P + \epsilon Q.
\end{align*}
In this model, we can define two subsets of $\mathcal{G}, \mathcal{B} \subseteq \{1,\ldots,n\}$,
where $i \in \mathcal{G}$ if the corresponding sample is drawn from $P$, and $i \in \mathcal{B}$ otherwise. 
Following \citep{diakonikolas2016robust}, we define a set of feasible weights as
\begin{align}
S_{n,\epsilon} = \left\{ \{w_1,\ldots,w_n\}: \sum_{i=1}^n w_i = 1, 0 \leq w_i \leq \frac{1}{(1 - 2\epsilon)n}~\forall~i~\right\}.
\end{align}
Noting that with high-probability there are fewer than $2\epsilon n$ points in the set $\mathcal{B}$, the
set $S_{n,\epsilon}$ with high-probability contains the ideal weights which we denote $w^*$ whose entries are given as,
\begin{align}
\label{eqn::goodwt}
w^*_i =  \frac{\mathbb{I}(i \in \mathcal{G})}{|\mathcal{G}|}~~~\forall~i.
\end{align}
For any given weight vector $w$, we define its renormalized restriction to the points in $\mathcal{G}$ via,
\begin{align*}
w^g_i = \frac{w_i}{\sum_{i \in \mathcal{G}} w_i}~~~\forall~i.
\end{align*}
With this notation in place, we can further define a collection of quantities of interest
for a fixed set of weights $w \in S_{n,\epsilon}$.
A naive estimator of the functional is simply 
\begin{align}
\widetilde{\theta}(w) = \sum_{i=1}^n w_i \theta_g (x_i),
\end{align}
and its error is denoted as $\widetilde{\Delta}(w) = \widetilde{\theta}(w) - \theta_g(P)$. 
A more nuanced estimator further exploits the expected sparsity of the functional
by truncating its smaller entries. We define, for a positive vector $v$, 
$P_k(v)$ to be the vector where the $k$-th largest entries are retained (breaking ties arbitrarily)
and all other entries are set to $0$. Then we define,
\begin{align}
\widehat{\theta}(w) = P_{2s}(\widetilde{w})
\end{align}
and its error $\widehat{\meandiff}\left(w\right) = \widehat{\theta}\left(w\right) - \theta_g\left(P\right)$.
Recalling, the definition of the covariance
functional in Eqn.~\eqref{eqn:cov} we define the error of the weighted covariance as,
\begin{align*}
\mathcal{E}(w) = \sum_{i=1}^n w_i (\func(x_i) - \theta_g(P))(\func(x_i) - \theta_g(P))^T - \text{cov}(\theta_g(P)). 
\end{align*}
In allowing for a high-dimensional scaling, where $d \gg n$, we can no longer expect
$\widetilde{\Delta}(w)$ to be small in an $\ell_2$ sense
and $\mathcal{E}(w)$ to be small in an operator norm sense. Instead, we rely on establishing
a more limited control on these quantities. We define the $s$-sparse operator norm as,
\begin{align*}
\sparseopnorm{M} = \max_{S \subset [d], |S| \leq s} \opnorm{M^{SS}}.
\end{align*}
Finally, we define $\|M\|_{\infty} = \max_{i,j} |M_{ij}|$.
With these definitions in place we can now state our main deterministic result.
We focus on functionals $\theta_g$ for which 
Equations~\eqref{eqn:cov_sparse_bound} and~\eqref{eqn:F_lip} are satisfied. 
\begin{thm}[Main Theorem]\label{thm:main}
Suppose that, for samples $\{x_1,\ldots,x_n\}$ drawn from the $\epsilon$-contamination model, we have that $\|\theta_g(x_i)\|_2 \leq D$, and further that 
there exist a universal constant $C_1$
such that the following conditions hold:
	\begin{align}
	\abs{\badind} &\leq 2\epsilon n\label{eqn:number_bad_points}, \\
	\|\widetilde{\Delta}(w^*)\|_{\infty} &\leq C_1 \left(\frac{\left(\lip_F + \sqrt{\lip_\cov}\right)\delta}{s}\right), \label{eqn:spar_mean_hat_tri_good_weights}\\
	\|P_{s}(\widetilde{\Delta}(w^g))\|_2 &\leq C_1\left(\left(\lip_F + \sqrt{\lip_\cov}\right)\delta\right)~~\forall~w~\in S_{n,\epsilon}, \label{eqn:spar_mean_hat_tri} \\
	\|\mathcal{E}(w^*)\|_{\infty} &\leq C_1\left(\frac{\left(\lip_F^2 + \lip_\cov\right)\delta}{s}\right),\label{eqn:spar_mean_square_good_weights} \\
	\sparseopnorm{\mathcal{E}(w^g)} &\leq C_1 \left(\left(\lip_F^2 + \lip_\cov\right)\delta\right)~~\forall~w ~\in S_{n,\epsilon}.\label{eqn:spar_mean_square}
	\end{align} 
	Then there is an algorithm which runs in time polynomial 
	in $\left(n,d,\frac{1}{\epsilon}\right)$ and outputs  $\widehat{\theta}$ satisfying $\|\widehat{\theta}-\theta\|_2 \leq C_2 \left(\left(\sqrt{L_\cov}+L_F\right)\delta\right)$ for some absolute constant $C_2$.
%Suppose $x^1,\cdots,x^n$ are sampled according to Eqn.~\eqref{eqn:huber_model}, the functional $\func$ satisfies Eqn.~\eqref{eqn:cov_sparse_bound}-~\eqref{eqn:F_lip} and $\norm{\func\left(X^i\right)}_2 \le D$ for all $i=1,\cdots,N$.
%	Let $w^\star$ denote the weights which are uniformly distributed on the uncorrupted points.
%	Suppose we have enough samples that the followings hold.
%	\begin{align}
%	\abs{\badind} & = \Theta\left(\epsilon N\right)\label{eqn:number_bad_points} \\
%	\norm{\covdiff\left(w^\star\right)}_{\infty,\infty} & = O\left(\frac{\left(\lip_F^2 + \lip_\cov\right)\delta}{s}\right).\label{eqn:spar_mean_square_good_weights} \\
%	\norm{\tilde{\meandiff}\left(w^\star\right)}_{\infty}  & = O\left(\frac{\left(\lip_F + \sqrt{\lip_\cov}\right)\delta}{s}\right). \label{eqn:spar_mean_hat_tri_good_weights}\\
%	\max_{\set{S}\subset[d],\abs{S}\le s}\norm{\covdiff\left(w^g\right)}_{2} & = O\left(\left(\lip_F^2 + \lip_\cov\right)\delta\right)  \text{ for } \forall w \in \set{S}_{N,\epsilon}, \label{eqn:spar_mean_square}\\ 
%	\max_{\set{S} \subset[d],\abs{S} \le s}\norm{\tilde{\meandiff}\left(w^g\right)}_{2}  & = O\left(\left(\lip_F + \sqrt{\lip_\cov}\right)\delta\right) \text{ for } \forall w \in \set{S}_{N,\epsilon}, \label{eqn:spar_mean_hat_tri}
%	\end{align} 
%	Then there is an algorithm which runs polynomially in $\left(d,\frac{1}{\epsilon},D\right)$ and outputs  $\hat{\func}$ satisfying $\norm{\hat{\func}-\func}_2 = O\left(\left(\sqrt{L_\cov}+L_F\right)\delta\right)$.
\end{thm}

Several remarks are in order. In order to apply the theorem to a specific statistical model, we simply need to verify that the functional is sufficiently regular (see Equations~\eqref{eqn:cov_sparse_bound} and~\eqref{eqn:F_lip}), that the functional is bounded by a polynomial in $(n,d,D,1/\epsilon)$, and finally that the conditions in Equations~\eqref{eqn:number_bad_points}-\eqref{eqn:spar_mean_square} are satisfied.  We ensure boundedness via a simple pruning step that removes gross, and easily detectable, outliers. In order to verify the main deviation conditions of the theorem, we note that there are two types of deviation we need to control. The first type in Equations~
\eqref{eqn:spar_mean_hat_tri_good_weights} and \eqref{eqn:spar_mean_square_good_weights} establishes strong $\ell_\infty$ control,
decaying with the sparsity $s$, but only needs to hold for the ideal weights $w^*$. The other type 
of control, in Equations~\eqref{eqn:spar_mean_hat_tri} and \eqref{eqn:spar_mean_square}
is on an $s$-sparse operator norm and needs to hold \emph{uniformly} over the set $S_{n,\epsilon}$, but 
importantly ignores the weights on the points in $\mathcal{B}$ via restriction to $w^g$. In concrete examples, we establish the latter control via the use of empirical process arguments (selecting an appropriate covering and using the union bound).
\subsection{Re-visiting Illustrative Examples}
We now turn our attention to the statistical problems introduced earlier,
and derive specific corollaries of our deterministic result.
The proofs of these results follow the recipe outlined previously, 
of verifying the various conditions and applying Theorem~\ref{thm:main}. 
We defer the technical details to the Appendix.
We begin with the case of estimating a sparse Gaussian mean, when the covariance
is the identity.
\begin{cor}[Robust Estimation of Sparse Gaussian Mean]\label{thm:sparse_mean}
Consider the model introduced in Equation~\eqref{eqn::mean}, then there 
are universal constants $C_1, C_2$ such that, if
$n \geq C_1\left(\frac{s^2\log\left(d/\tau\right)}{\epsilon^2\log1/\epsilon}\right),$
then there exists an algorithm that runs in time polynomial in 
$\left(d,n\right)$ and outputs
an estimate~$\widehat{\mu}$ that with probability at least  $1-\tau$ satisfies: \[
	\norm{\widehat{\mu}-\mu}_2^2 \leq C_2 \epsilon^2\log\frac{1}{\epsilon}.
	\]
%	Let $P = \mathcal{N}\left(\mu,\mat{I}\right)$ with $\norm{\mu}_0 \le s$.
%	If $N = \Omega\left(\frac{s^2\log\left(d/\tau\right)}{\epsilon^2\log1/\epsilon}\right)$, then there is an algorithm that runs polynomial in $\left(d,n,\log\left(1/\tau\right)\right)$ and  outputs $\hat{\mu}$ that with probability $1-\tau$ satisfies: \[
%	\norm{\hat{\mu}-\mu}_2^2 = O\left(\epsilon^2\log\frac{1}{\epsilon}\right).
%	\]
\end{cor}
It is worth noting that in contrast to prior work the sample complexity, has a logarithmic dependence on the ambient dimension $d$, allowing for high-dimensional scalings where $d \gg n$, provided that
the sparsity 
$s^2 \ll n$. As in the work of \citet{diakonikolas2016robust}, we obtain near-optimal contamination dependence scaling upto a logarithmic factor as roughly $\epsilon^2$. Importantly, as emphasized in prior work \citep{diakonikolas2016robust,lai2016agnostic} 
and in stark contrast to other tractable robust estimators, the contamination dependence achieved by our algorithm is completely independent of the dimension of the problem.

In comparing to information-theoretic lower bounds (see Appendix~\ref{sec:sparse_mean}), we notice that the sample complexity is worse by a factor $s$. As will be clearer in the sequel, this increased sample
complexity is due to use of a convex relaxation for sparse PCA \citep{d2007direct}. This phenomenon, arises in a variety of statistical estimation problems and is believed to be related to the hardness of the planted clique problem~\citep{berthet2013computational}. Next, we consider the performance of our method,
in estimating a sparse covariance matrix. 
\begin{cor}[Robust Sparse Gaussian Covariance Estimation]\label{thm:sparse_cov}
Consider the model introduced in Equation~\eqref{eqn::covmodel}.
There are universal constants $C_1,C_2$ such that if 
the sample size $n \geq C_1 \left(\frac{s^2\log\left(d/\tau\right)}{\epsilon^2}\right)$, then there is an algorithm that runs in time polynomial in $\left(d,n\right)$ and produces
an estimate $\widehat{\Omega}$ that with probability at least $1-\tau$ satisfies: \[
	\fronorm{\widehat{\Omega} - \Omega}^2 \leq C_2 \left(\fronorm{\Omega}^2 \epsilon^2\log^4\frac{1}{\epsilon}\right).
	\]
\end{cor}
We note that once again, the result is applicable even when $n \ll d$, that the statistical estimation rate
is optimal upto a factor of $s$ and that the contamination dependence is optimal upto logarithmic factors. We recall that, as discussed previously, we can combine 
by appropriate reductions the previous two corollaries in order
to obtain a high-dimensional robust estimator for a Gaussian with unknown mean and covariance. 
Lastly, we apply our estimator to the various generalized linear models introduced earlier.
\begin{cor}[Robust Sparse Generalized Linear Models]\label{thm:sparse_linear}
Consider the models in Equations~\eqref{eqn::regmodel},\eqref{eqn::glmmodel}, and~\eqref{eqn::logisticmodel}. If the target parameter $\beta$ satisfies, $\|\beta\|_2 \leq \rho$,
then there exist universal constants $C_1, C_2$ such that if 
$n \geq C_1 \left(\frac{s^2\log\left(d/\tau\right)}{\epsilon^2}\right)$, then there exists an algorithm that runs in time polynomial in $\left(d,n,\rho\right)$ and produces an estimate $\widehat{\beta}$ such that with probability at least $1 - \tau$:
\begin{enumerate}
\item  Linear and Generalized Linear Models:  $\|\widehat{\beta} - \beta\|_2^2 \leq C_2 \left(\left(\norm{\beta}_2^2+1\right) \epsilon^2\log^4\frac{1}{\epsilon}\right).$
\item Logistic-type Models:  $\|\widehat{\beta} - \beta\|_2^2 \leq C_2 \left(\left(\norm{\beta}_2^2+1\right) \epsilon^2\log^2\frac{1}{\epsilon}\right).$
\end{enumerate}

%
%	Suppose $P$ is a joint distribution of $\left(X,y\right)$ that \[
%	X\sim \mathcal{N}\left(0,\mat{I}\right), y = \beta^\top X + \xi \text{ where } \xi \sim \mathcal{N}\left(0,1\right).
%	\]
%	with $ \norm{\beta}_2 \le \rho, \norm{\beta}_0 \le s$.
%	and \[\left(X^1,y_1\right),\cdots,\left(X^n,y_n\right) \sim \left(1-\epsilon\right)P + \epsilon Q\] where $B$ is any known upper bound of $\beta$ and $Q$ is an arbitrary distribution.
%	If number of samples satisfies $n = \Omega\left(\frac{s^2\log\left(d/\tau\right)}{\epsilon^2}\right)$, then there is an algorithm that runs polynomial in $\left(d,n,\log\left(1/\tau\right),\rho\right)$ and outputs $\hat{\beta}$ that with probability $1-\tau$ satisfies\[
%	\norm{\beta-\hat{\beta}}_2^2 = O\left(\epsilon^2\log^4\frac{1}{\epsilon}\left(\norm{\beta}_2^2+1\right)\right).
%	\]
\end{cor}
%We treat the logistic and generalized linear models using slightly different techniques in the Appendix. 
%In particular, 
By exploiting the natural boundedness of the logistic-type models, we are able to obtain slightly stronger guarantees than in the regression setting. 
Taken together, the results in this section provide the first guarantees on
 computationally tractable methods for robust estimation which achieve near-optimal contamination dependence in the sparse high-dimensional setting.

\section{Unified Algorithm and Technical Insights}
\label{sec:techniques}
Broadly, our main algorithm follows the template of the convex programming approach 
for Gaussian mean estimation in \citet{diakonikolas2016robust}, described in Algorithm~\ref{algo:sparse_func}. The algorithm proceeds in two steps, first a naive pruning step
is applied to remove clear outliers in order to ensure that various quantities remain bounded
by a radius that is polynomial in $(n,d,1/\epsilon)$. 
In the sequel, we use $\{z_1,\ldots,z_m\}$ to denote the pruned sample. 
We generalize a similar pruning step
from prior works \citep{diakonikolas2016robust,lai2016agnostic} to deal with the generalized
linear model settings. 
This in turn further 
ensures that the subsequent use of the ellipsoid algorithm, terminates in polynomial time. At a high-level
the ellipsoid algorithm is used to exploit the covariance structure of the functional 
in order to obtain a weighting of the sample that appropriately
down-weighs detrimental samples from the contamination distribution $Q$.

\paragraph{Separation oracle via sparse PCA: } Our first main technical contribution, is a new separation oracle, appropriate for the high-dimensional 
setting. The separation oracle in the work of \citet{diakonikolas2016robust} is
based on the operator norm
deviation between the weighted empirical covariance from its known or anticipated form.
Roughly, the intuition is that a good weighting of the samples must induce a good estimate
of the covariance function. In the high-dimensional setting when $n \ll d$, even
in the absence of outliers the
covariance function 
cannot be estimated well in the 
operator norm. Exploiting the sparsity of the underlying functional we show that it suffices instead
to ensure that the weighted empirical covariance is close to its anticipated form only on
$s$-sparse subsets of the coordinates, which indeed does hold in a variety of statistical models
even in the high-dimensional scaling. 
However, this leads to the next technical hurdle: to design the separation oracle we need 
to be able to detect the deviation of the weighted empirical covariance on sparse subsets.
This is the sparse PCA problem and is known to be NP-hard in a strong sense~\citep{tillmann2014computational}. We consider instead using a well-known 
convex relaxation for sparse PCA \citep{d2007direct}, and via a novel analysis show 
upto a loss of a factor of $s$ in the sample complexity, this convex relaxation suffices
to construct our separation oracle. Notably, the use of sparse PCA enables high-dimensional estimation
of sparse functionals.

%In order to exploit the sparsity of the underlying functional, we 
%notice that 
%
%because we need number of sample $n$ at least linear in $d$.
%Again to exploit the sparsity structure, we only require on every $s$ coordinates $\set{S}$ of $\func_{\set{S}}$, the corresponding empirical matrix is close to the truth.
%However, this approach needs to examine every $s$-subset of $[d]$, which requires exponential time in $s$. 
%In fact, this is precisely the sparse PCA problem.
%However, the exact computation is NP-hard in the strong sense~\citep{tillmann2014computational}. 
%Our second contribution is that we propose to use the convex program for sparse PCA introduced by~\cite{d2007direct}.
%We show that as long as the empirical covariance is close to the truth in the $\ell_\infty$ norm, this convex relaxation does not affect the inference accuracy by much.

\begin{algorithm}[tb]
	\caption{Robust Sparse Functional Estimation}
	\label{algo:sparse_func}
	\begin{algorithmic}[1]
		\STATE \textbf{Input: }$\left\{x_1,\ldots,x_n \right\}$, $\tau_{\text{prune}}$, $s$, $\tau_{\text{sep}}$
		\STATE Run a naive pruning algorithm, with input $(\left\{x_1,\ldots,x_n \right\}, 
		\tau_{\text{prune}})$ and output $\{z_1,\ldots,z_m\}$.
		\STATE Run the ellipsoid algorithm using the separation oracle described in Algorithm~\ref{algo:sepra_func} with input $\left(\left\{z_1,\ldots,z_m\right\},s,\tau_{\text{sep}}\right)$ and output $\left\{w_1,\ldots,w_m\right\}.$
		\STATE \textbf{Output: $\widehat{\theta} = \tops\left(\sum_{i=1}^{m} w_ig\left(z_i\right)\right)$}.	
		\end{algorithmic}
\end{algorithm}

%Our algorithm follows a similar design pattern as the convex programming approach for Gaussian mean estimation in~\citep{diakonikolas2016robust}.
%We first review the key ingredients of their algorithm.
%The first step is to throw away samples that are unlikely to be sampled from $P$.
%This can be done using our prior knowledge about $P$, like the known covariance of Gaussian mean estimation problem.
%The second step is to obtain a set of weights $w_1,\cdots,w_n$ for each data point such that the plug-in estimator $\hat{\func} = \sum_{i=1}^{n}w_i\func\left(X^i\right)$ is close to the true functional.
%Note that such weights always exists if $n$ is large because we can simply put a uniform distribution over $X^i$ with $X^i \sim P$.
%To obtain these good weights, \cite{diakonikolas2016robust} proposed to ellipsoid algorithm which at each iteration it refines the current set of weights and showed that after polynomial number of iterations, the algorithm is able to find a set of weights that makes the plug-in estimator work.
%The last component is to design a separation oracle for ellipsoid algorithm.
%\cite{diakonikolas2016robust} uses the largest eigen-pair of the difference matrix between the plug-in estimator of the covariance $\sum_{i=1}^{N}w_i\left(\func\left(X^i\right)-\hat{\func}\right)\otimes \left(\func\left(X^i\right)-\hat{\func}\right)$ and the true covariance matrix to construct such separation oracle.

\paragraph{\bf Hard-thresholding with redundancy: } In the remainder of this section we highlight two other technical contributions
that enable a unified treatment of a large set of sparse functionals. 
Even in the absence of outliers, the natural estimator for a functional -- its empirical counterpart -- is inconsistent when $n \ll d$, at least in an $\ell_2$ sense. However, even in the high-dimensional setting the empirical estimator, remains adequate both in an $\ell_\infty$ sense, and over sparse subsets. In settings of interest to ensure for instance that uniformly over all $s$-sparse subsets the empirical estimator is close to its population counterpart, it suffices to collect roughly $s \log d$ samples.
In order to exploit this insight when the true functional is sparse we use a careful truncation at various points in order to establish appropriate error control. A key aspect of this truncation is to ensure a certain redundancy by retaining roughly twice as many entries at each step, which allows us to adequately control the possible bias induced by truncation.

%In high dimensional setting the plug-in estimator does not work directly for Euclidean norm because in order to have $\norm{\func - \sum_{i=1}^{n}\frac{1}{n}\func\left(X^i\right)}_2 \le \delta$, we need $n \asymp \frac{d}{\delta^2}$ samples even in the no contamination case.
%%However, if we know that the functional is sparse, then we can modify our the naive plug-in estimator to exploit this structure.
%However, on any subset of coordinates $\set{S}$ with $\abs{\set{S}} = O\left(s\right)$, the plug-in estimator still performs well, i.e.,  to achieve $\norm{\func_{\set{S}} - \sum_{i=1}^{n}\frac{1}{n}\func_{\set{S}}\left(X^i\right)}_2 \le \delta$, we only need $n \asymp \frac{s\log d}{\delta^2}$.
%Based on this concentration fact and our prior knowledge that $\func$ is $s$-sparse, we propose to use a post-propcessing procedure $\tops\left(\cdot\right)$ on top of the plug-in estimator $\sum_{i=1}^{n}\frac{1}{n}\func\left(X^i\right)$.
%$\tops\left(\nu\right)$ only keeps the largest $2s$ entries of $\nu$ and set other entries to $0$.
%Note we keep top $2s$ instead of top $s$ entries.
%Such redundancy ensure entries in the support of $\func$ with relatively large value are kept in our estimator.

\paragraph{General forms for the covariance: } A final conceptual contribution that we highlight is 
generalizing the basic insight of the work of \citet{diakonikolas2016robust}. At a high-level, a key observation of their work is that in cases where the covariance structure is either 
known or in some sense related to the mean structure, this fact can be exploited in order to identify
good weightings of the samples. Roughly, good weighting schemes induce good covariance estimates, while conversely bad weightings induce noticeably poor covariance estimates. We generalize
this insight, identifying a set of smoothness conditions on the covariance map (see Equations~\eqref{eqn:cov_sparse_bound} and \eqref{eqn:F_lip}) that allow us to tractably exploit the covariance structure. Concretely, deriving the covariance structure for mean and covariance estimation, 
GLMs and logistic-type models and showing that they satisfy 
these conditions enables a unified treatment.

%
%
%Note in~\citep{diakonikolas2016robust}, for Gaussian covariance and other estimation problems, since the covariance of the functional may not be known, the authors design sophisticated methods to solve these problems.
%A general idea is sometime the covariance can be written as an analytical expression $F\left(\func\right)$ (c.f. Eqn.~\eqref{eqn:F}).
%Therefore, we can use this expression and our estimated $\hat{\func}$ to obtain a good approximation of the covariance.
%Our third contribution is that we show sufficient condition (Eqn.~\eqref{eqn:F_lip}) of this expression making our estimator work .
%This idea is particular useful for sparse linear regression and sparse covariance estimation problems

%\paragraph{Convex Relaxation for Sparse PCA} The empirical covariance can not be estimated even in operator norm because we need number of sample $n$ at least linear in $d$.
%Again to exploit the sparsity structure, we only require on every $s$ coordinates $\set{S}$ of $\func_{\set{S}}$, the corresponding empirical matrix is close to the truth.
%However, this approach needs to examine every $s$-subset of $[d]$, which requires exponential time in $s$. 
%In fact, this is precisely the sparse PCA problem.
%However, the exact computation is NP-hard in the strong sense~\citep{tillmann2014computational}. 
%Our second contribution is that we propose to use the convex program for sparse PCA introduced by~\cite{d2007direct}.
%We show that as long as the empirical covariance is close to the truth in the $\ell_\infty$ norm, this convex relaxation does not affect the inference accuracy by much.

\begin{algorithm}[tb]
	\caption{Separation Oracle for Robust Sparse Estimation}
	\label{algo:sepra_func}
	\begin{algorithmic}[1]
		\STATE \textbf{Input:} Weights from the previous
		iteration $\left\{w_1,\ldots,w_m\right\}$, pruned samples $\{z_1,\ldots,z_m\}$, tolerance parameter $\tau_{\text{sep}},$ sparsity level $s$.
		\STATE Compute $\widehat{\theta} = \tops\left(\sum_{i=1}^{m}w_i\func\left(z_i\right)\right)$.
		\STATE Compute 
		$E = \sum_{i=1}^{m} w_i \left(\func\left(z_i\right)-\widehat{\theta}\right) \otimes \left(\func\left(z_i\right)-\widehat{\theta}\right)- F(\widehat{\theta})$.
		\STATE Solve the following convex program \footnotemark:
		\begin{equation}
		\begin{aligned}
		\max_{\mat{H}} &\quad  \tr\left(\mat{E}\mat{H}\right) \\
		\text{subject to } &  \mat{H} \succcurlyeq 0 \\
		& \norm{H}_{1,1} \le s \\
		& \tr\left(\mat{H}\right) = 1.
		\end{aligned}\label{eqn:convex_relax}
		\end{equation}
		 $\mat{H}^*$ be the solution and $\lambda^*$ be the optimal value.
		\IF {$\lambda^* \le \tau_{\text{sep}}$}
		\STATE \textbf{Return:} ``Yes''.
		\ELSE 
		\STATE \textbf{Return:} The separating hyperplane: \begin{align*}
		\ell(w') = \tr \Big(\Big[\Big(\sum_{i=1}^{m}w_i'\big(\func\left(z_i\right)-\widehat{\theta}\big)\otimes\big(\func\left(z_i\right)-\widehat{\theta}\big) \Big]- F\big(\widehat{\theta}\big)\Big)\mat{H}^*\Big) - \lambda^*.
		\end{align*} 
		\ENDIF
	\end{algorithmic}
\end{algorithm}
\footnotetext{As will be clearer in the proofs, it suffices to solve this program to accuracy roughly $O(\epsilon)$.}

\subsection{Analysis of the Main Algorithm}

\label{sec:estimator}
In the remainder of this section, we describe a few key theoretical results in the analysis 
of our main algorithm. 

\paragraph{Hard Thresholding: }\label{sec:tops}
The idea of using hard thresholding in sparse estimation problems 
in order to ensure that the overall estimation error is well controlled, has been explored 
recently 
in iterative hard thresholding algorithms (see for instance \cite{bhatia2015robust,jain2014iterative}). 
The key result we need, relates the $s$-sparse subset error of the original estimator to the full $\ell_2$
error of the hard-thresholded estimator. Recalling the definitions of 
the error of the original estimator $\widetilde{\Delta}$ and the error of the thresholded estimator
$\widehat{\Delta}$ we show the following result:
\begin{lem}~\label{thm:top_2s_lemma}
Suppose $\theta_\func$ is $s$-sparse, then we have the following result:
	\[
	\frac{1}{5}\|\widehat{\meandiff}\|_2 \le \norm{P_s\left(\widetilde{\meandiff}\right)}_2 \le 4\|\widehat{\meandiff}\|_2.
	\]
\end{lem}
We defer the proof of this result to the Appendix. Intuitively, this result lets us pass from the high-dimensional feasible error control on subsets to the more desirable $\ell_2$ error control.

\paragraph{Good Weights and Approximation of the Covariance: }
\label{sec:good_weights}
The utility of the ellipsoid algorithm is in finding an appropriate set of weights, such that the weighted empirical estimate of the functional is sufficiently accurate. 
%Now we turn our attention to finding a good set of weights making $\hat{\func}$ a good estimate.
%Inspired by~\citep{diakonikolas2016robust}, 
In more details, we consider weights such that the weighted covariance is close to the true one on every sparse subset of coordinates. Defining, $\widehat{\theta} = \tops\left(\sum_{i=1}^{m}w_i\func\left(z_i\right)\right)$:
\begin{defn} \label{def:good_weights}[Good Weights] Let $C_\delta$ be the 
subset of $S_{m,\epsilon}$ such that for any $w \in C_\delta$ we have 
\begin{align*}
\sparseopnorm{\sum_{i=1}^{m} w_i \big(\func\left(z_i\right)-\widehat{\theta}\big)\big(\func(z_i)-\widehat{\theta}\big)^\top-\cov\left(\theta_\func\right)} \le \left(\lip_F^2 + \lip_\cov\right)\delta.
\end{align*}
\end{defn}
The parameter $\delta$ in the above definition is an accuracy parameter that will be chosen as a function of only $\epsilon$ differently for each model. The central role of this set of weights is 
captured by the following result, whose proof follows along similar lines to that 
of Lemma 4.19 of~\citet{diakonikolas2016robust}. 
%We define, $S^* = \argmax_{S \subset [d],|S| \le s} \|\widetilde{\meandiff}_S\|_2$. 
%Intuitively, the lemma guarantees that
%the error in estimating the functional 

%We define good weights in terms of empirical covariance because we can use the operator norm of the error of our estimation of covariance to construct a separation oracle (c.f. Algorithm~\ref{algo:sepra_func}). 
%The following lemma shows if we find $w \in C_\delta$, then the plug-in estimator $\tilde{\func}$ is close to the true functional on any subset of coordinates with cardinality smaller than $s$.
%The proof is similar to Lemma 4.19 of~\citep{diakonikolas2016robust}.
\begin{lem}\label{thm:cov_imply_mean}
Let $w \in S_{m,\epsilon}$ and suppose that for a universal constant $C_1$ we have, 
\begin{align*}
\abs{\badind} &\leq 2 \epsilon n,\\
\|P_{s}(\widetilde{\Delta}(w^g))\|_2 &\leq C_1\left(\left(\lip_F + \sqrt{\lip_\cov}\right)\delta\right), \\
	\sparseopnorm{\mathcal{E}(w^g)} &\leq C_1 \left(\left(\lip_F^2 + \lip_\cov\right)\delta\right),
%\max_{\set{S}\subset[d],\abs{S}\le s}\norm{\covdiff\left(w^g\right)}_{2} & \le C_1\left(\lip_F^2 + \lip_\cov\right)\delta,\\ 
%\max_{\set{S} \subset[d],\abs{S} \le s}\norm{\tilde{\meandiff}\left(w^g\right)}_{2}  & \le C_1\left(\lip_F + \sqrt{\lip_\cov}\right)\delta
\end{align*} 
where $\delta \ge C_2\epsilon$ for some sufficiently large constant $C_2$.
If $\norm{P_s\left(\widetilde{\meandiff}\left(w\right)\right)}_2 \ge C_3\left(\lip_F + \sqrt{\lip_\cov}\right)\delta$  for some sufficiently large constant $C_3$, then for sufficiently small $\epsilon$ we have that,
\begin{align*}
 \sparseopnorm{\sum_{i=1}^{m}w_i\big(\func\left(z_i\right)-\widehat{\theta}\big)\big(\func\left(z_i\right)-\widehat{\theta}\big)^\top-\cov\left(\theta_\func\right)} \ge  \frac{\norm{P_s\left(\widetilde{\meandiff}\left(w\right)\right)}_2^2}{4\epsilon}.
	\end{align*}
\end{lem}
Roughly, this lemma guarantees that if the weighting scheme is such that the error $\widetilde{\Delta}^S$ is large (in $\ell_2$) then the weights cannot belong to the set of good weights defined above. 
We note that an essentially identical result can be proved if we replace the true covariance by a plug-in estimate, provided the covariance map is sufficiently smooth (see Lemma~\ref{thm:cov_imply_mean_approx} in the Appendix). 
This results in an important reduction, in order to obtain an accurate estimate it suffices to find a weight vector that belongs to the set of good weights. We accomplish this via the ellipsoid algorithm. 

%\subsection{Approximation of $\cov\left(\func\right)$}
%\label{sec:cov_approx}
%In many situations we do not know $\cov\left(\func\right)$ but only a function $F\left(\func\right) = \cov\left(\func\right)$.
%The following lemma shows we can replace $\cov\left(\func\right)$ in Lemma~\ref{thm:cov_imply_mean} by $F\left(\hat{\func}\right)$, an approximation using $F$ and our estimation $F\left(\hat{\theta}\right)$.
%The proof just relies on Lemma~\ref{thm:cov_imply_mean} and our assumptions on $F$.
%\begin{lem}\label{thm:cov_imply_mean_approx}
%Using the same notations and assuming the same conditions as Lemma~\ref{thm:cov_imply_mean}, we have \begin{align*}
%\norm{\sum_{i=1}^{N}w_i\left(\func_{\set S}\left(Y^i\right)-\hat{\func}_{\set S}\right)\otimes\left(\func_{\set S}\left(Y^i\right)-\hat{\func}_{\set S}\right)^\top - F\left(\hat{\func}_{\set S}\right)}_{op} \ge  \frac{\norm{\tilde{\meandiff}_\set{S}}_2^2}{5\epsilon}.
%\end{align*}
%\end{lem}

\paragraph{Convex Relaxation of Sparse PCA: }
\label{sec:convex_relax}
In order to use the previous lemma in the ellipsoid algorithm, we need to be able to 
design a separation oracle for the set of good weights. The main technical hurdle is
that we need to compute, for a given set of weights, the sparse operator norm which is 
an intractable problem in general~\citep{magdon2015np}.

We replace the sparse PCA algorithm by a standard tractable 
convex relaxation~\citep{d2007direct}. The following result shows 
that the optimal value of this program is sandwiched by the optimal value of the 
intractable sparse PCA program. In the Appendix, we provide a complete analysis
of our algorithm and show that this approximation suffices to ensure soundness and completeness
of the ellipsoid algorithm.

\begin{lem}\label{thm:convex_relax}
For a fixed $w$, the optimal value $\lambda^*\left(w\right)$ of Eqn.~\eqref{eqn:convex_relax} satisfies
\begin{align*}
\lambda^*\left(w\right) \ge \sparseopnorm{\sum_{i=1}^{m}w_i\big(\func\left(z_i\right)-\widehat{\theta}\left(w\right)\big)\big(\func\left(z_i\right)-\widehat{\theta}(w)\big)^\top - F\big(\widehat{\theta}(w)\big)}.
\end{align*}
Furthermore, the solution $H^*(w)$ satisfies that there is a universal constant $C$ such that for any $w' \in S_{m,\epsilon}$
\begin{align*}
 &\tr \Big(\Big(\sum_{i=1}^{m}w_i' \big(\func\left(z_i\right)-\widehat{\theta}(w)\big)\big(\func\left(z_i\right)-\widehat{\theta}\left(w\right)\big)- F(\widehat{\theta}\left(w\right)) \Big)\mat{H}^*(w)\Big)\\
\leq & C\left(s\norm{\mathcal{E}(w')}_{\infty} + \norm{\widehat{\meandiff}\left(w\right)}_2^2 + \left(L_F+ s\norm{\widetilde{\meandiff}\left(w'\right)}_{\infty}\right)\norm{\widehat{\meandiff}\left(w\right)}_2\right).
\end{align*}
\end{lem}
Concretely, the above lemma provides two guarantees. First that the optimal value
of the relaxation is never too small, so that
the ellipsoid algorithm does not falsely accept a bad weighting scheme, and finally, that the separating
hyperplane is sufficiently accurate when appropriate control can be established on the various stochastic fluctuations. We combine these two facts to complete the analysis of the ellipsoid algorithm and to establish Theorem~\ref{thm:main} in the Appendix.

\section{Conclusion and Future Directions}
\label{sec:con}
In this paper we propose a computationally tractable 
robust algorithm for sparse high-dimensional 
statistical estimation problems. We develop a general result, which we then specialize
to obtain corollaries for sparse mean/covariance estimation, sparse linear regression and
sparse generalized linear models. In each of these problems, we obtain near optimal dependency
on the contamination parameter, and sample complexities that depend only logarithmically 
on the ambient dimension. 

Future directions of research include developing faster alternatives to the ellipsoid algorithm, to further relax the Gaussian assumption in various settings, and finally to close the gap in sample complexity to statistically optimal, albeit computationally intractable procedures \citep{chen2015robust,chen2016general}.

\bibliographystyle{abbrvnat}
\bibliography{simonduref}

\appendix

\section{Proofs of Sec.~\ref{sec:estimator}}
\label{sec:proofs}
\subsection{Proof of Hard Thresholding}
\begin{lem}[Lemma \ref{thm:top_2s_lemma}]
Suppose $\theta_\func$ is $s$-sparse, then we have the following result:
\[
\frac{1}{5}\|\widehat{\meandiff}\|_2 \le \norm{P_s\left(\widetilde{\meandiff}\right)}_2 \le 4\|\widehat{\meandiff}\|_2.
\]
\end{lem}
\begin{proof}
We denote $\set{S}^*$ to be the support for $\theta$ and $\set{S}$ be indices for the selected $2s$ entries.
We first prove $\frac{1}{5}\norm{\widehat{\meandiff}}_2 \le \norm{P_s\left(\widetilde{\meandiff}\right)}_2$.
Let $\tau = \norm{P_s\left(\widetilde{\meandiff}\right)}_2$.
We have
\begin{align*}
\norm{\widehat{\theta}-\theta}_2 \le \norm{\widehat{\theta}^{\set{S}} - \theta^{\set{S}}_\func}_2 
+\norm{\widehat{\theta}^{\set{S}^c \cap \left(\set{S}^*\right)} - \theta^{\set{S}^c\cap \left(\set{S}^*\right)}_\func }_2
+\norm{\widehat{\theta}^{\set{S}^c \cap \left(\set{S}^*\right)^c} - \theta^{\set{S}^c\cap \left(\set{S}^*\right)^c}_\func}_2.
\end{align*}
Now we bound the three terms in the right hand side separately.
The first term is bounded by $2\tau$ by our assumption.
The third term is $0$ by definition of $\set{S}$ and $\set{S}^*$.
For the second term, note \begin{align*}
	\norm{\widehat{\theta}^{\set{S}^c \cap \left(\set{S}^*\right)} - \theta_\func^{\set{S}^c\cap \left(\set{S}^*\right)} }_2 & \le 
	\norm{\widetilde{\theta}^{\set{S}^c \cap \left(\set{S}^*\right)} - \theta_\func^{\set{S}^c\cap \left(\set{S}^*\right)} }_2 + 
	\norm{\widetilde{\theta}^{\set{S}^c \cap \left(\set{S}^*\right)} - \widehat{\theta}^{\set{S}^c\cap \left(\set{S}^*\right)} }_2 \\
	& = \norm{\widetilde{\theta}^{\set{S}^c \cap \left(\set{S}^*\right)} - \theta_\func^{\set{S}^c\cap \left(\set{S}^*\right)}}_2 + 
	\norm{\widetilde{\theta}^{\set{S}^c \cap \left(\set{S}^*\right)} }_2.
\end{align*}
We have $\norm{\widetilde{\theta}^{\set{S}^c \cap \left(\set{S}^*\right)} - \theta_\func^{\set{S}^c\cap \left(\set{S}^*\right)} }_2 \le \tau$ by assumption.
Assume $\norm{\widetilde{\theta}^{\set{S}^c \cap \left(\set{S}^*\right)} }_2 \ge 2\tau$.
Since $\abs{\set{S}} = 2s$, then there exists $\set{S}' \subset \set{S}$, $\abs{\set{S}'}= s$ such that $\set{S}^* \cap \set{S}' = \emptyset$ with $\norm{\widetilde{\theta}^{\set{S}'}}_2 \ge \norm{\widetilde{\theta}^{\set{S}^c \cap \set{S}^*}}_2 \ge 2\tau$.
However,  $\norm{\widetilde{\theta}^{\set{S}'}}_2 = \norm{\widetilde{\meandiff}^{\set{S}'}}_2 \le \tau$ by our assumption.
Therefore,  $\norm{\widetilde{\theta}^{\set{S}^c \cap \left(\set{S}^*\right)} }_2 \le 2\tau$.
Adding all these terms up, we have $\norm{\widehat{\theta}-\theta_\func}_2 \le 4\tau$.

For the other direction, let $\gamma =\norm{\widehat{\meandiff}}_2$.
For any $\set{S}' \subset [d]$, $\abs{\set{S}'}\le s$, we have \begin{align*}
\norm{\widetilde{\meandiff}^{\set{S}'}}_2 & \le \norm{\widetilde{\meandiff}^{\set{S}' \cap \set{S}}}_2 + \norm{\widetilde{\meandiff}^{\set{S}' \cap \set{S}^c}}_2 \\
& = \norm{\widehat{\meandiff}^{\set{S'}\cap \set{S}}}_2 + \norm{\widetilde{\meandiff}^{\set{S}' \cap \set{S}^c}}_2 \\
& \le \gamma + \norm{\widetilde{\meandiff}^{\set{S}' \cap \set{S}^c}}_2
\end{align*} where the last inequality is by our assumption.
Now applying triangle inequality on $\norm{\widetilde{\meandiff}^{\set{S}' \cap \set{S}^c}}_2$, we have 
\begin{align*}
	\norm{\widetilde{\meandiff}^{\set{S}' \cap \set{S}^c}}_2 & \le  \norm{\widetilde{\meandiff}^{\set{S}' \cap \set{S}^c \cap \set{S}^*}}_2 
	+ \norm{\widetilde{\meandiff}^{\set{S}' \cap \set{S}^c \cap \left(\set{S}^*\right)^c}}_2.
\end{align*}
For the first term, observe that \begin{align*}
	\norm{\widetilde{\meandiff}^{\set{S}' \cap \set{S}^c \cap \set{S}^*}}_2  & \le \norm{\widetilde{\meandiff}^{\set{S}^c \cap \set{S}^*}}_2 \\
	& = \norm{\widetilde{\theta}^{\set{S}^c \cap \set{S}^*} - \theta_\func^{\set{S}^c \cap \set{S}^*}}_2 \\
	& \le \norm{\widetilde{\theta}^{\set{S}^c \cap \set{S}^*}}_2 +\norm{ \theta_\func^{\set{S}^c \cap \set{S}^*}}_2 \end{align*}
By definition of $\tops$, there exists $\set{S}'' \subset \set{S}$ with $\abs{\set{S}''} = s $ and $\theta^{\set{S}''} = 0$ and $\norm{\widetilde{\theta}^{\set{S}''}}_2 \ge \norm{\widetilde{\theta}^{\set{S}^c \cap \set{S}^*}}_2$.
Therefore, $\norm{\widetilde{\theta}_{\left(\set{S}\right)^c \cap \set{S}^*}}_2 \le \gamma$.
Next, notice $\norm{ \theta_\func^{\set{S}^c \cap \set{S}^*}}_2 = \norm{\widehat{\triangle}^{\set{S}^c\cap \set{S}^*}}_2 \le \gamma$.
Lastly, note again $\abs{\set{S}' \cap \left(\set{S}\right)^c \cap \left(\set{S}^*\right)^c} \le s$, so\begin{align*}
	\norm{\widetilde{\meandiff}^{\set{S}' \cap \set{S}^c \cap \left(\set{S}^*\right)^c}}_2 
	= &\norm{\widetilde{\theta}^{\set{S}' \cap \set{S}^c \cap \left(\set{S}^*\right)^c}}_2 \\
	\le & \norm{\widetilde{\theta}^{\set{S}''}}_2 \\
	= &\norm{\widetilde{\meandiff}^{\set{S}''}}_2 \\
	= & \norm{\widehat{\meandiff}^{\set{S}''}}_2 \\
	\le &\gamma.
	\end{align*}
Therefore $\widetilde{\meandiff}^{\set{S}'} \le 4\gamma$.
Because $\set{S}'$ is arbitrary, our proof is complete.
\end{proof}

\subsection{Proofs of Good Weights and Approximation of the Covariance}
\begin{lem}[Lemma~\ref{thm:cov_imply_mean}]
Let $w \in S_{m,\epsilon}$ and suppose that for a universal constant $C_1$ we have, 
\begin{align*}
\abs{\badind} &\leq 2 \epsilon n,\\
\|P_{s}(\widetilde{\Delta}(w^g))\|_2 &\leq C_1\left(\left(\lip_F + \sqrt{\lip_\cov}\right)\delta\right), \\
	\sparseopnorm{\mathcal{E}(w^g)} &\leq C_1 \left(\left(\lip_F^2 + \lip_\cov\right)\delta\right),
%\max_{\set{S}\subset[d],\abs{S}\le s}\norm{\covdiff\left(w^g\right)}_{2} & \le C_1\left(\lip_F^2 + \lip_\cov\right)\delta,\\ 
%\max_{\set{S} \subset[d],\abs{S} \le s}\norm{\tilde{\meandiff}\left(w^g\right)}_{2}  & \le C_1\left(\lip_F + \sqrt{\lip_\cov}\right)\delta
\end{align*} 
where $\delta \ge C_2\epsilon$ for some sufficiently large constant $C_2$.
If $\norm{P_s\left(\widetilde{\meandiff}\left(w\right)\right)}_2 \ge C_3\left(\lip_F + \sqrt{\lip_\cov}\right)\delta$  for some sufficiently large constant $C_3$, then for sufficiently small $\epsilon$ we have that,
\begin{align*}
 \sparseopnorm{\sum_{i=1}^{m}w_i\big(\func\left(z_i\right)-\widehat{\theta}\big)\big(\func\left(z_i\right)-\widehat{\theta}\big)^\top-\cov\left(\theta_\func\right)} \ge  \frac{\norm{P_s\left(\widetilde{\meandiff}\left(w\right)\right)}_2^2}{4\epsilon}.
	\end{align*}
\end{lem}
\begin{proof}
Let $\set{S} = \argmax_{\set{S}' \subset [d], \abs{\set{S'}} \le s} \norm{\meandiff^{\set{S}'}}_2$.
Assumptions in the lemma imply that
\begin{align*}
\norm{\sum_{i\in\badind} w_i\left(\func^{\set{S}}\left(z_i\right) - \theta_\func^{\set S}\right)}_2 &= \norm{\widetilde{\meandiff}^{\set{S}}-\sum_{i\in\goodind} w_i\left(\func^{\set{S}}\left(z_i\right) - \theta_\func^{\set S}\right)}_2 \\
& = \frac{C_1}{5}\left(\lip_F + \sqrt{\lip_\cov}\right)\delta - c\left(\lip_F + \sqrt{\lip_\cov}\right)\delta \\
& = \left(\frac{C_1}{5}-c\right)\left(\lip_F + \sqrt{\lip_\cov}\right)\delta 
\end{align*}
where we have used Lemma~\ref{thm:top_2s_lemma}.
Now consider the covariance.
We have
\begin{align*}
	\sum_{i\in\badind}\frac{w_i}{w_b}\left(\func^{\set{S}}\left(z_i\right)-\theta_\func^{\set{S}}\right)  \left(\func^{\set{S}}\left(z_i\right)-\theta_\func^{\set{S}}\right)  \succcurlyeq \left(\frac{w_i}{w_b}\left(\func^{\set{S}}\left(z_i\right)-\theta_\func^{\set{S}}\right)\right)  \left(\frac{w_i}{w_b}\left(\func^{\set{S}}\left(z_i\right)-\theta_\func^{\set{S}}\right)\right) 
\end{align*}
because of the non-negativity of variance.
Therefore, because $\abs{\badind} \le 2\epsilon n$, we have \[
\norm{\sum_{i\in \badind}\left(\func^{\set{S}}\left(z_i\right)-\theta_\func^{\set{S}}\right)  \left(\func^{\set{S}}\left(z_i\right)-\theta_\func^{\set{S}}\right)^\top}_{\footnotesize{\mbox{op}}} \ge  \frac{\norm{\widetilde{\meandiff}^{\set{S}}}_2^2}{2\epsilon}.
\]
Now using our assumption on the covariance, we have\begin{align*}
&\norm{\sum_{i=1}^{m}w_i\left(\func^{\set{S}}\left(z_i\right)-\theta_\func^\set{S}\right)\left(\func^\set{S}\left(z_i\right)-\theta_\func^\set{S}\right)^\top}_{\footnotesize{\mbox{op}}} \\
\ge & 		\norm{\sum_{i\in\badind}w_i\left(\func^{\set{S}}\left(z_i\right)-\theta_\func^\set{S}\right)\left(\func^\set{S}\left(z_i\right)-\theta_\func^\set{S}\right)^\top}_{\footnotesize{\mbox{op}}} - \\ &
\norm{\sum_{i\in\goodind}w_i\left(\func^{\set{S}}\left(z_i\right)-\theta_\func^\set{S}\right)\left(\func^\set{S}\left(z_i\right)-\theta_\func^\set{S}\right)^\top - w_g\cov\left(\theta_\func^\set{S}\right)}_{\footnotesize{\mbox{op}}} -\norm{w_b\cov\left(\theta_\func^{\set{S}}\right)}_{\footnotesize{\mbox{op}}} \\
= &\frac{\norm{\widetilde{\meandiff}^{\set{S}}}^2}{2\epsilon} -c\left(\lip_F^2+\lip_\cov\right)\delta-2\epsilon\lip_{\cov}\\
\ge &\frac{\norm{\widetilde{\meandiff}^{\set{S}}}^2}{3\epsilon}
\end{align*} where in the last inequality we have used the assumption that $\epsilon$ is sufficiently small.
Lastly, we use the expression 
\begin{align*}
& \sum_{i=1}^{m}w_i\left(\func^{\set{S}}\left(z_i\right)-\widehat{\theta}^{\set{S}}\right)\left(\func^{\set{S}}\left(z_i\right)-\widehat{\theta}^{\set{S}}\right)^\top-\cov\left(\theta_\func^\set{S}\right) \\
=& \sum_{i=1}^{m}w_i\left(\func^{\set{S}}\left(z_i\right)-\theta_\func^\set{S}\right)\left(\func^\set{S}\left(z_i\right)-\theta_\func^\set{S}\right)^\top - \cov\left(\theta_\func^{\set{S}}\right) -\widehat{\meandiff}^{\set{S}}\left(\widetilde{\meandiff}^{\set{S}}\right)^\top -\widetilde{\meandiff}^{\set{S}}\left(\widehat{\meandiff}^{\set{S}}\right)^\top 
+ \widehat{\meandiff}^{\set{S}}\left(\widehat{\meandiff}^{\set{S}}\right)^\top.
\end{align*}
to obtain
\begin{align*}
&\norm{\sum_{i=1}^{m}w_i\left(\func^{\set{S}}\left(z_i\right)-\widehat{\theta}^{\set{S}}\right)\left(\func^{\set{S}}\left(z_i\right)-\widehat{\theta}^{\set{S}}\right)^\top-\cov\left(\theta_\func^\set{S}\right)}_{\footnotesize{\mbox{op}}} \\
\ge & \norm{\sum_{i=1}^{m}w_i\left(\func^{\set{S}}\left(z_i\right)-\theta_\func^\set{S}\right)\left(\func^\set{S}\left(z_i\right)-\theta_\func^\set{S}\right)^\top - \cov\left(\theta_\func^{\set{S}}\right)}_{\footnotesize{\mbox{op}}} - 24\left(\norm{\widetilde{\meandiff}^{\set{S}}}_2^2\right) \\
\ge &\frac{\norm{\widetilde{\meandiff}^{\set{S}}}^2}{4\epsilon}.
\end{align*} 
\end{proof}

\begin{lem}\label{thm:cov_imply_mean_approx}
Using the same notations and assuming the same conditions as Lemma~\ref{thm:cov_imply_mean}, 
we have
\begin{align*}
 \sparseopnorm{\sum_{i=1}^{m}w_i\left(\func\left(z_i\right)-\widehat{\theta}\left(w\right)\right)\left(\func\left(z_i\right)-\widehat{\theta}\left(w\right)\right)^\top-F\left(\widehat{\theta}\left(w\right)\right)} \ge  \frac{\norm{P_s\left(\widetilde{\meandiff}\left(w\right)\right)}_2^2}{5\epsilon}.
	\end{align*}
\end{lem}
\begin{proof}
With the same notations in the proof of Lemma~\ref{thm:top_2s_lemma}, we know \begin{align*}
\norm{\sum_{i=1}^{m}w_i\left(\func^{\set{S}}\left(z_i\right)-\widehat{\theta}^{\set{S}}\right)\left(\func^{\set{S}}\left(z_i\right)-\widehat{\theta}^{\set{S}}\right)^\top-\cov\left(\theta_\func^\set{S}\right)}_{\footnotesize{\mbox{op}}}  \ge  \frac{\norm{\widetilde{\meandiff}^{\set{S}}}^2}{4\epsilon}.
\end{align*}
By our assumptions on $F$,  we have \begin{align*}
\sparseopnorm{F\left(\theta\right)-F\left(\widehat{\theta}\right)} \le & L_F\norm{\widehat{\meandiff}}_2 + C\norm{\widehat{\meandiff}}_2^2 \\
\le& 5L_F\norm{\widetilde{\meandiff}^\set{S}}_2 + 5C\norm{\widetilde{\meandiff}^\set{S}}_2^2.
\end{align*}
Since $\delta = \Omega\left(\epsilon\right)$, $\epsilon$ is larger than any absolute constant, applying triangle inequality to previous two inequalities, we obtain the desired result.
\end{proof}

\subsection{Proofs of Convex Relaxation of Sparse PCA}\label{sec:convex_proof}
\begin{thm}[Theorem~\ref{thm:convex_relax}]
For a fixed $w$, the optimal value $\lambda^*\left(w\right)$ of Eqn.~\eqref{eqn:convex_relax} satisfies
\begin{align*}
\lambda^*\left(w\right) \ge \sparseopnorm{\sum_{i=1}^{m}w_i\big(\func\left(z_i\right)-\widehat{\theta}\left(w\right)\big)\big(\func\left(z_i\right)-\widehat{\theta}(w)\big)^\top - F\big(\widehat{\theta}(w)\big)}.
\end{align*}
Furthermore, the solution $H^*(w)$ satisfies that there is a universal constant $C$ such that for any $w' \in S_{m,\epsilon}$
\begin{align*}
 &\tr \Big(\Big(\sum_{i=1}^{m}w_i' \big(\func\left(z_i\right)-\widehat{\theta}(w)\big)\big(\func\left(z_i\right)-\widehat{\theta}\left(w\right)\big)- F(\widehat{\theta}\left(w\right)) \Big)\mat{H}^*(w)\Big)\\
\leq & C\left(s\norm{\mathcal{E}(w')}_{\infty} + \norm{\widehat{\meandiff}\left(w\right)}_2^2 + \left(L_F+ s\norm{\widetilde{\meandiff}\left(w'\right)}_{\infty}\right)\norm{\widehat{\meandiff}\left(w\right)}_2\right).
\end{align*}
\end{thm}
\begin{proof}
Because this is a convex relaxation of sparse PCA, the lower bound is naturally satisfied.
For the upper bound, again we use the decomposition\begin{align*}
&\sum_{i=1}^{m}w_i'\left(\func\left(z_i\right)-\widehat{\theta}\right)\left(\func\left(z_i\right)-\widehat{\theta}\right)-  F\left(\widehat{\theta}\right)\\
= &\covdiff\left(w'\right) - \widetilde{\meandiff}\left(w'\right)\widehat{\meandiff}\left(w\right)^\top - \widehat{\meandiff}\left(w\right)\widetilde{\meandiff}\left(w'\right)^\top + \widehat{\meandiff}\left(w\right)\widehat{\meandiff}\left(w\right)^\top + \cov\left(\func\right) - F\left(\widehat{\theta}\left(w\right)\right).
\end{align*}
First applying \holder inequality on trace we have \[
\tr\left(\covdiff\left(w'\right)H^*\left(w\right)\right) \le \norm{\covdiff\left(w'\right)}_{\infty,\infty}\norm{H^*\left(w\right)}_{1,1} \le s\norm{\covdiff\left(w'\right)}_{\infty,\infty}.
\]
Similarly, we have\begin{align*}
&\tr\left(\left(
 \widetilde{\meandiff}\left(w'\right)\widehat{\meandiff}\left(w\right)^\top + \widehat{\meandiff}\left(w\right)\widetilde{\meandiff}\left(w'\right)^\top
 \right)H^*\left(w'\right)\right) \\
\le & 2s\norm{\widetilde{\meandiff}\left(w'\right)}_\infty\norm{\widehat{\meandiff}\left(w\right)}_2.
\end{align*}
Note $H^*\left(w\right)$ belongs to the Fantope $\mathcal{F}^1$ ~\citep{overton1992sum,vu2013fantope}, so\begin{align*}
&\tr\left( \widehat{\meandiff}\left(w\right)\widehat{\meandiff}\left(w\right)^\top H^*\left(w\right) \right) 
\le \norm{\widehat{\meandiff}\left(w\right)\widehat{\meandiff}\left(w\right)^\top}_{\footnotesize{\mbox{op}}} \le \norm{\widehat{\meandiff}\left(w\right)}_2^2.
\end{align*}
Using this property again we have
\begin{align*}
\tr\left(\left[\cov\left(\func\right)-F\left(\widehat{\theta}\left(w\right)\right)\right]H^*\left(w\right)\right) \le & \norm{\cov\left(\func\right)-F\left(\widehat{\theta}\left(w\right)\right)}_{\footnotesize{\mbox{op}}}\\
& \le L_F\norm{\widehat{\meandiff}\left(w\right)}_2 + C\norm{\widehat{\meandiff}\left(w\right)}_2^2.
\end{align*}
Putting these together we obtain the desired result.
\end{proof}

\subsection{Proofs of Ellipsoid Algorithm}
We begin with proving the correctness of the separation oracle.
\begin{thm} [Separation Oracle]\label{thm:correct_sep}
Let $w^*$ denote the weights which are uniform on the uncorrupted points.
Suppose Eqn.~\eqref{eqn:number_bad_points}-Eqn.~\eqref{eqn:spar_mean_hat_tri_good_weights} hold, then there exists a sufficiently large absolute constant $C_{good}$ that if we set $\tau_{sep} = \Omega\left( \left(\lip_F^2+\lip_\cov\right)\delta\right)$, Algorithm~\ref{algo:sepra_func} satisfies\begin{enumerate}
\item (\emph{Completeness}) If $w = w^*$, the algorithm outputs ``Yes". 
\item (\emph{Soundness}) If $w \notin C_{C_{good}\left(\lip_F^2+\lip_{\cov}\right)\delta}$, the algorithm outputs a hyperplane $\ell\left(\cdot\right)$ such that $\ell\left(w\right) \ge 0$.
Moreover, if the algorithm ever outputs a hyperplane $\ell$, then $\ell\left(w^*\right) < 0.$
\end{enumerate}
\end{thm}
\paragraph{Remark:} The conditions of this separation oracle is slightly weaker than the traditional ones. 
However, note that outside $C_{C_{good}\left(\lip_F^2+\lip_{\cov}\right)\delta}$, the separation oracle acts exactly as a separation oracle for $w^*$.
\begin{proof}
First, for the completeness, plugging Eqn.~\eqref{eqn:spar_mean_hat_tri} and Eqn.~\eqref{eqn:spar_mean_square} into Theorem~\ref{thm:convex_relax} and then using Lemma~\ref{thm:top_2s_lemma}, we directly obtain the desired result.
If $w \notin C_{C_{good}\left(\lip_F^2+\lip_{\cov}\right)\delta}$, we can apply the lower bound in Theorem~\ref{thm:convex_relax} and use Lemma~\ref{thm:cov_imply_mean_approx}.
See Lemma~\ref{thm:w_bad_output_plane} for the full proof.
When the algorithm outputs a hyperplane, $\ell\left(w\right) \ge 0$ follows directly by the optimality of the convex program.
Lastly, we use the upper bound of Theorem~\ref{thm:convex_relax} to argue $\ell\left(w^*\right) < 0$ whenever we outputs a hyperplane (Lemma~\ref{thm:plane_for_w_star}). 
\end{proof}
%\begin{lem}\label{thm:w_star_no_plane}
%Let $H\left(w\right)$ be the solution of the convex relaxed spare PCA program with weights $w$. 
%If $w = w^*$, then $\tr\left(H^*\left(w^*\right)G\left(w^*\right)\right) = O\left(\delta\right)$.
%\end{lem}
%\begin{proof}
%Plugging Eqn.~\eqref{eqn:spar_mean_square}-~\eqref{eqn:spar_mean_hat_tri} into Theorem~\ref{thm:convex_relax} and using Lemma~\ref{thm:top_2s_lemma}, we directly obtain the desired result.
%\end{proof}
%Second, the following lemma shows if $w$ is not in the good set then $\lambda^*$ is big and we will proceed to construct the hyperplane.
\begin{lem}\label{thm:w_bad_output_plane}
	If $w \notin \set{C}_{C_{good}\left(\lip_F^2+\lip_\cov\right)\delta}$, then $\lambda^* = \Omega\left(\left(\lip_F^2+\lip_\cov\right)\delta\right)$.
\end{lem}
\begin{proof}
Applying the lower bound of Theorem~\ref{thm:convex_relax}, we have
	\begin{align*}
		& \tr\left(H^*\left(w\right)G\left(w\right)\right) \\
		\ge & \sparseopnorm{\sum_{i=1}^{m}w_i\left(\func\left(z_i\right)-\widehat{\theta}\left(w\right)\right)\left(\func\left(z_i\right)-\widehat{\theta}\left(w\right)\right)^\top - F\left(\widehat{\theta}\right)}.
	\end{align*}
Now if $\norm{\widehat{\meandiff}}_2 \ge 5C_1 \left(\lip_F+\sqrt{\lip_\cov}\right)\delta$ where $C_1$ is defined in Lemma~\ref{thm:cov_imply_mean},  by Lemma~\ref{thm:cov_imply_mean_approx} and Lemma~\ref{thm:top_2s_lemma}, we have \begin{align*}
&\sparseopnorm{\sum_{i=1}^{m}w_i\left(\func\left(z_i\right)-\widehat{\theta}\left(w\right)\right)\left(\func\left(z_i\right)-\widehat{\theta}\left(w\right)\right)^\top - F\left(\widehat{\theta}\right)} \\
\ge &\frac{\norm{\widehat{\meandiff}}_2^2}{5\epsilon} \\
= &\Omega\left(\left(\lip_F^2 + \lip_{\cov}\right)\delta\right).
\end{align*}
On the other hand if $\norm{\widehat{\meandiff}}_2 \le 5C_1 \left(\lip_F+\sqrt{\lip_\cov}\right)\delta$, by Lemma~\ref{thm:top_2s_lemma},  by definition of $\set{C}_{C_{good}\left(\lip_F^2+\lip_\cov\right)\delta}$, we have\begin{align*}
&\tr\left(H^*\left(w\right)M\left(w\right)\right) \\
\ge & C_{good}\left(\lip_F^2 +\lip_\cov \right)\delta -\sparseopnorm{F\left(\theta\right) - F\left(\widehat{\theta}\right)} \\
\ge  & C_{good}\left(\lip_F^2 +\lip_\cov \right)\delta  - \lip_F\norm{\widehat{\meandiff}}_2 - C\norm{\widehat{\meandiff}}_2^2 \\
= &\Omega\left(\left(\lip_F^2 +\lip_\cov \right)\delta\right)
\end{align*} where the last step we use the fact that $C_{good}$ is large enough.
\end{proof}
%The following two lemmas show the soundness of the separation oracle.
%\begin{lem}\label{thm:plane_for_itself}
%For any hyperplane produced by $w$, $\ell\left(w\right) \ge 0$.
%\end{lem}
%\begin{proof}
%	The lemma follows from the definition of $\lambda^*$.
%\end{proof}
\begin{lem}\label{thm:plane_for_w_star}
	For any hyperplane $\ell$, $\ell\left(w^*\right) < 0$.
\end{lem}
\begin{proof}
We apply the upper bound of Theorem~\ref{thm:convex_relax} with $w'=w^*$.
Therefore, we only need to upper bound\begin{align*}
O\left(s\norm{\covdiff\left(w'\right)}_{\infty,\infty} + \norm{\widehat{\meandiff}\left(w\right)}_2^2 + \left(L_F+ s\norm{\widetilde{\meandiff}\left(w'\right)}_{\infty}\right)\norm{\widehat{\meandiff}\left(w\right)}_2\right) - \lambda^*\left(w\right) < 0.
\end{align*}
Plugging in our assumptions on $w^*$, we just need to show\begin{align*}
C_2\left(\left(\lip_F^2 + \lip_\cov\right)\delta + \left(\left(\lip_F + \sqrt{\lip_\cov}\right)\delta+\lip_F\right)\norm{\widehat{\meandiff}\left(w\right)}_2 + \norm{\widehat{\meandiff}\left(w\right)}_2^2\right) - \lambda^*\left(w\right) < 0
\end{align*} for some absolute constant $C_2$.
If $\norm{\widehat{\meandiff}\left(w\right)}_2 \ge 5C_1 \left(\lip_F + \sqrt{L_\cov}\right)\delta$ for $C_1$ defined in Lemma~\ref{thm:cov_imply_mean}, using the argument in Lemma~\ref{thm:w_bad_output_plane}, we know $\lambda^*\left(w\right) = \Omega\left(\frac{\norm{\widehat{\meandiff}}_2^2}{\epsilon}\right)$.
Therefore, we have
\begin{align*}
\ell\left(w^*\right) &\le C_2\left(\left(\lip_F^2 + \lip_\cov\right)\delta + \left(\left(\lip_F + \sqrt{\lip_\cov}\right)\delta+\lip_F\right)\norm{\widehat{\meandiff}\left(w\right)}_2 + \norm{\widehat{\meandiff}\left(w\right)}_2^2\right) -  \Omega\left(\frac{\norm{\widehat{\meandiff}\left(w\right)}_2^2}{\epsilon}\right)\\
& \le C_2\left(\left(\lip_F^2 + \lip_\cov\right)\delta + \left(\left(\lip_F + \sqrt{\lip_\cov}\right)\delta+\lip_F\right)\norm{\widehat{\meandiff}\left(w\right)}_2\right) -  \Omega\left(\frac{\norm{\widehat{\meandiff}\left(w\right)}_2^2}{\epsilon}\right) \\
& \le C_2\lip_F\norm{\widehat{\meandiff}\left(w\right)}_2 - \Omega\left(\frac{\norm{\widehat{\meandiff}\left(w\right)}_2^2}{\epsilon}\right) \\
& < 0
\end{align*}
where the second equality we used $\norm{\widehat{\meandiff}\left(w\right)}_2 \ge 5C_1 \left(\lip_F + \sqrt{L_\cov}\right)\delta$ and the third we used the fact that $\delta = \Omega\left(\epsilon\right)$.
If $\norm{\widehat{\meandiff}\left(w\right)}_2 \le 5C_1 \left(\lip_F + \sqrt{\lip_\cov}\right)\delta$, since $\lambda^*\left(w\right) \ge \tau_{sep} \ge C_3\left(\lip_F^2 + \lip_\cov\right)\delta$ for $C_3$ sufficiently large, we have
\begin{align*}
\ell\left(w^*\right) &\le C_2\left(\left(\lip_F^2 + \lip_\cov\right)\delta + \left(\left(\lip_F + \sqrt{\lip_\cov}\right)\delta+\lip_F\right)\norm{\widehat{\meandiff}\left(w\right)}_2 + \norm{\widehat{\meandiff}\left(w\right)}_2^2\right) -  C_3\left(\lip_F^2 + \lip_\cov\right)\delta \\
& = - \Omega\left(\left(\lip_F^2 + \lip_\cov\right)\delta\right) < 0.
\end{align*}
Thus, whenever we output a hyperplane $\ell$, $\ell\left(w^*\right) < 0$.
\end{proof}

Now, by classical convex programming result, after polynomial iterations we can obtain $w$ such that there exists $w' \in C_{C_{good}\left(\lip_F^2+\lip_{\cov}\right)\delta}$, $\norm{w-w'}_{\infty} \le \frac{\epsilon\left(\sqrt{\lip_{\cov}}+\lip_F\right)}{nD}$.
Lemma~\ref{thm:last_step} shows this $w$ is good enough to make $\widehat{\theta_g}\left(w\right)$ a good estimate.
This finishes the proof of Theorem~\ref{thm:main}.
\begin{lem}\label{thm:last_step}
Given $w$, if there exists $w'\in C_{C_{good}\left(\lip_F^2+\lip_{\cov}\right)\delta}$ such that $\norm{w-w'}_\infty \le \frac{\epsilon\left(\sqrt{\lip_{\cov}}+\lip_F\right)}{mD}$, then $\norm{\widehat{\meandiff}\left(w\right)}_2 = O\left(\left(\sqrt{\lip_\cov}+\lip_F\right)\delta\right).$
\end{lem}
\begin{proof}
By the assumptions, we have\begin{align*}
	\norm{\widehat{\meandiff}\left(w\right)}_2 \le & 5\norm{P_s\left(\widetilde{\meandiff}_S\left(w\right)\right)}_2 \\
	= & 5 \norm{P_s\left(\widetilde{\meandiff}\left(w'\right) + \sum_{i=1}^{m}\left(w_i-w_i'\right)\left(\func\left(z_i\right)-\theta_\func\right)\right)}_2 \\
	\le & 5\norm{P_s\left(\meandiff\left(w'\right)\right)}_2 + \sum_{i=1}^{m} \abs{w_i-w_i'}\norm{\func\left(z_i\right)-\theta_\func}_2 \\
	=& O\left(\left(\sqrt{\lip_\cov}+\lip_F\right)\delta\right) + \left(\sqrt{\lip_{\cov}} + \lip_F\right)\epsilon \\
	= & O\left(\left(\sqrt{\lip_\cov}+\lip_F\right)\delta\right).
	\end{align*}
\end{proof}

\section{Technical Details of Sparse Mean Estimation}
\label{sec:sparse_mean}
In this section we prove Theorem~\ref{thm:sparse_mean}. 
Since $F\left(\func\right) = \mat{I}$, a constant function, we know $L_{\cov} = 1$ and $L_F=0$.
We adopt Algorithm 1 in~\citep{diakonikolas2016robust} to achieve the boundedness condition in Theorem~\ref{thm:main}.
The pseudocodes are listed in Algorithm~\ref{algo:prune_mean} for completeness.
Maximal inequality of Gaussian random variables shows with probability $1-\tau$, this procedure does not remove any example sampled from $P$.
\begin{algorithm}[tb]
	\caption{Naive Pruning for Gaussian Mean}
	\label{algo:prune_mean}
	\begin{algorithmic}[1]
		\STATE \textbf{Input:} $\left\{x_1,\cdots,x_n\right\}$ 
		\STATE \textbf{For} $i,j=1,\cdots,n$, let $\delta_{ij} = \norm{x_i-x_j}_2$.
		\FOR{$i=1,\cdots,j$}
		\STATE Let $A_i = \left\{j\in 1,\cdots, n: \delta_{ij} = \Omega\sqrt{d\log\left(n/\tau\right)}\right\}$
		\IF {$\abs{A_i} > 2\epsilon n$}
		\STATE remove $x_i$ from the set.
		\ENDIF
		\ENDFOR
	\end{algorithmic}
\end{algorithm}

Now we prove the concentration inequalities in Theorem~\ref{thm:main}.
Note when $n = \Omega\left(\frac{s^2\log\left(d/\tau\right)}{\epsilon^2}\right)$ Eqn.~\eqref{eqn:number_bad_points} -~\eqref{eqn:spar_mean_hat_tri_good_weights} can be proved through classical Bernoulli and Gaussian concentration inequalities.
For the remaining two, we use the following lemma.
\begin{lem}\label{thm:sparse_worst_sample_concentration}
	Fix $0 < \epsilon < 1/2$ and $\tau < 1$.
	There is a $\delta = O\left(\epsilon\sqrt{\log\left(1/\epsilon\right)}\right)$ such that if $x_1,\cdots,x_n \sim N\left(\mu,\mat{I}\right)$ and $n = \Omega\left(\frac{s\log d + \log \left(1/\tau\right)}{\delta^2}\right)$ then for any $w \in S_{n,\epsilon}$ the followings hold: \begin{align}
	\max_{\norm{v}_0\le s, \norm{v}_2 \le 1} \abs{v^\top \left(\sum_{i=1}^{n}w_i\left(x_i-\mu\right)\left(x_i-\mu\right)^\top - \mat{I}\right)v} & \le \delta \label{eqn:mean_concen_cov} \\
	\max_{\set{S} \subset [d],\abs{\set{S}} \le s} \norm{\sum_{i=1}^{n}w_i\left(x_i^{\set{S}}-\mu\right)}_2 & \le \delta. \label{eqn:mean_concen_mean}
	\end{align}
\end{lem}
\begin{proof}
The proof is similar to Lemma 4.5 of~\citep{diakonikolas2016robust}.
We prove the concentration result for Eqn.~\eqref{eqn:mean_concen_cov}, Eqn,~\eqref{eqn:mean_concen_mean} follows similarly by replacing the classical concentration inequality of covariance by that of mean.
Without loss of generality, we assume $\mu=0$.
For any $J \subset [n]$, $\abs{J} = \left(1-2\epsilon\right)n$, we let $w^J$ be the vector which is given by $w_i^J=\frac{1}{\abs{J}}$ for $i \in J$ and $w_i^{J} = 0$ otherwise.
By convexity, it suffices to show that \begin{align*}
\prob\left[\forall J \subset [n]: \abs{J}=\left(1-2\epsilon\right)n \text{ and } \max_{\set{S}\subset[d], \abs{S}\le s}\norm{\sum_{i=1}^{n}w_i^Jx_i^{\set{S}}\left(x_i^{\set{S}}\right)^\top -\mat{I}}_{2}\ge \delta\right]\le\tau.
\end{align*}
We first fix $\tau'$, $\set S \subset [d]$ with $\abs{S} \le s$ and $J \subset [n]$.
Using triangle inequality we have \begin{align*}
&\norm{\sum_{i=1}^{n}w_i^Jx_i^{\set{S}}\left(x_i^{\set{S}}\right)^\top - \mat{I}}_{op} \\ \le &\norm{\frac{1}{\left(1-2\epsilon\right)n}\sum_{i=1}^{n}x_i^{\set{S}}\left(x_i^{\set{S}}\right)^\top - \frac{1}{\left(1-2\epsilon\right)n}\mat{I}}_{op} + \norm{\frac{1}{\left(1-2\epsilon\right)n}\sum_{i \notin J}x_i^{\set S}\left(x_i^{\set S}\right)^\top - \frac{2\epsilon}{1-2\epsilon}\mat{I}}_{op}.
\end{align*}
By classical concentration bound, the first term is small than $\frac{\delta}{2}$ with probability at least $1-\frac{\tau'}{2}$ if $n = \Omega\left(\frac{s + \log\left(1/\tau'\right)}{\delta^2}\right).$
Similarly, the second term is smaller than $\delta/2$ with probability at least $1-\frac{\tau'}{2}$ if $n = \Omega\left(\frac{\epsilon \left(s+\log\left(1/\tau'\right)\right)}{\delta^2}\right)$.
Now by union bound over all subset $\set{S} \subset [d]$ with $\abs{\set{S}}\le s$ we have if $n = \Omega\left(\frac{s\log d +\log\left(1/\tau\right)}{\delta^2}\right)$ \[
\norm{\frac{1}{\left(1-2\epsilon\right)n}\sum_{i=1}^{n}x_i^{\set{S}}\left(x_i^{\set{S}}\right)^\top - \frac{1}{\left(1-2\epsilon\right)n}\mat{I}}_{op} \le \frac{\delta}{2}.
\]
Similarly, if $n = \Omega\left(\frac{\epsilon\left(s\log d + \log\left(1/\tau'\right)\right)}{\delta^2}\right)$ we have \[
\norm{\frac{1}{\left(1-2\epsilon\right)n}\sum_{i \notin J}x_i^{\set S}\left(x_i^{\set S}\right)^\top - \frac{2\epsilon}{1-2\epsilon}\mat{I}}_{op} \le \frac{\delta}{2}.
\]
Now choosing $\tau' = \begin{pmatrix}
n \\
\left(1-2\epsilon\right)n
\end{pmatrix}^{-1}\tau$ and taking union bounds over all $J$, by our choice of $\delta$ and $n$ in the theorem we have \[
\norm{\frac{1}{\left(1-2\epsilon\right)n}\sum_{i \notin J}x_i^{\set S}\left(x_i^{\set S}\right)^\top - \frac{2\epsilon}{1-2\epsilon}\mat{I}}_{op} \le \frac{\delta}{2}.
\]
Our proof is complete.
\end{proof}
We accompany our upper bound with the following minimax lower bound.
\begin{thm}[Lower Bound of Sparse Gaussian Mean Estimation]
	There are some constants $C,c$ such that \begin{align*}
	\inf_{\widehat{\mu}}\sup_{x\sim N\left(\mu,\mat{I}\right),\norm{\mu}_0\le s} \sup_{Q} \prob\left[\norm{\widehat{\mu}-\mu}_2^2 \ge C\left(\frac{s\log\left(d\right)}{n} \vee \epsilon^2\right)\right] \ge c.
	\end{align*}
\end{thm}
\begin{proof}
	First, the minimax lower bound for no adversary is $\asymp \frac{s\log d}{n}$.
	Further we know there exist $\mu_1$ and $\mu_2$ with $\norm{\mu_1}_0,\norm{\mu_2}_0 \le s$ and $\text{TV}\left(N\left(\mu_1,\mat{I}\right),N\left(\mu_2,\mat{I}\right)\right) \le \frac{2\epsilon}{1-2\epsilon}$ such that $\norm{\mu_1 -\mu_2}_2^2 \ge C'\epsilon$ (just consider two vectors each has only one non-zero entry).
	Now apply Theorem 4.1 of~\citep{chen2015robust}.
\end{proof}

\section{Technical Details of Sparse Covariance Estimation}
\label{sec:sparse_cov}
In this section we prove Theorem~\ref{thm:sparse_cov}.
%Let $\nodiag\left(\mat{A}\right)$ denote setting all diagonal entries of $\mat{A}$ to $0$.
%Using the zero mean assumption, we can simply set $\func\left(P\right) = \expect_{x\sim\mathcal{N}\left(0,\Sigma\right)}\left[\nodiag\left(xx^\top\right)\right]$.
By Theorem 4.15 of~\citep{diakonikolas2016robust} we have the following formula for the $\cov\left(\Omega\right)$\[
F\left(\Omega\right) = \Omega \otimes \Omega + \vect{\Omega} \otimes \vect{\Omega}.
\]
Now observe that $\tr\left(\Sigma\right) = d$ so for $x_1,\cdots,x_n \sim \mathcal{N}\left(0,\Sigma\right)$, using maximal inequality of Gaussian random variables, we have  \[
\prob\left[\max_i\norm{x_i}_2 \ge \Omega\left(d\sqrt{\log\left(N/\tau\right)}\right)\right] \le \tau.
\]
Therefore we can apply Algorithm~\ref{algo:prune_covariance} to achieve the boundedness assumption in Theorem~\ref{thm:main}.
Lastly, for the concentration bounds, note that 
Eqn.~\eqref{eqn:number_bad_points}, Eqn.~\eqref{eqn:spar_mean_hat_tri_good_weights} and Eqn.~\eqref{eqn:spar_mean_square_good_weights} can be proved by polynomial of Gaussian random variables and Eqn.~\eqref{eqn:spar_mean_hat_tri} and Eqn.~\eqref{eqn:spar_mean_square} are simple corollaries of Theorem 4.17 of~\cite{diakonikolas2016robust} with a union bound over subsets of $[d]$ with cardinality $s$.
\begin{algorithm}[tb]
	\caption{Pruning for Sparse Covariance}
	\label{algo:prune_covariance}
	\begin{algorithmic}[1]
		\STATE \textbf{Input:} $\left\{x_1,\cdots,x_n\right\}$ 
		\FOR{$i=1,\cdots,n$}
		\IF {$\norm{x_i}_2 = \Omega\left(d\sqrt{\log\left(n/\tau\right)}\right)$.}
		\STATE remove $x_i$ from the set.
		\ENDIF
		\ENDFOR
	\end{algorithmic}
\end{algorithm}

\section{Technical Details of Sparse Linear Regression}
\label{sec:sparse_linear}
In this section we study the sparse linear regression problem.
We begin by investigating the basic properties of our model.
\begin{thm}\label{thm:sparse_linear_basics}
	If \[
	x\sim N\left(0,\mat{I}\right), y = x\beta + \xi \text{ where } \xi \sim N\left(0,1\right),
	\] then we have\begin{align*}
	\expect \left[yx\right] &= \beta \\
	\cov\left[yx\right] &= \left(\norm{\beta}_2^2+1\right)\mat{I} + \beta\beta^\top
	\end{align*}
\end{thm}
\begin{proof}
	We first look at the expectation.
	\begin{align*}
	\expect\left[yx\right] = & \expect\left[x\left(\beta^\top x + \xi\right)\right] \\
	& = \expect\left[x x^\top\right]\beta + \expect\left[x\right]\expect\left[\xi\right] \\
	& = \beta.
	\end{align*}
	For the covariance note \begin{align}
	\cov\left[yx\right] = \expect\left[y^2x x^\top\right] - \beta\beta^\top.
	\end{align}
	We expand the first term.
	\begin{align*}
	\expect\left[y^2xx^\top\right]
	= & \expect\left[x\left(x^\top \beta + \xi\right)\left(\beta^\top x +\xi\right)x^\top\right] \\
	= & \expect\left[
	xx^\top \beta\beta^\top xx^\top
	\right] + \expect\left[\xi^2xx^\top\right] \\
	= & \expect\left[
	xx^\top \beta\beta^\top xx^\top
	\right] + \mat{I}.
	\end{align*}
	where we have used the independence of $x$ and $\xi$ to cancel out the cross terms.
	Now consider a single coordinate of $\expect\left[
	xx^\top \beta\beta^\top xx^\top
	\right]$, using Isserlis's theorem we have\begin{align*}
	\expect\left[
	e_i^\top xx^\top \beta\beta^\top xx^\top e_j
	\right] &= 2\expect\left[e_i^\top xx^\top \beta\right]\expect\left[\beta^\top x x^\top e_j\right] + \expect\left[e_i^\top x x^\top e_j\right]\expect\left[\beta^\top x x^\top \beta\right] \\
	= & \begin{cases}
	2\beta_i^2 + \norm{\beta}_2^2 &\text{ if } i=j \\
	2\beta_i \beta_j &\text{ if } i \neq j. 
	\end{cases}
	\end{align*}
	Note this implies \[
	\expect\left[
	xx^\top \beta\beta^\top xx^\top
	\right] = \norm{\beta}_2^2\mat{I} + 2\beta\beta^\top.
	\]
	Therefore, we have \[
	\cov\left[yx\right] = \left(\norm{\beta}_2^2+1\right)\mat{I} + \beta\beta^\top.
	\]
\end{proof}
With these expressions at hand, it is easy to upper bound $\lip _F$ and $\lip_\cov$.
\begin{cor}\label{thm:sparse_linear_lips}
	Under the same assumptions as Theorem~\ref{thm:sparse_linear_basics}, we have \begin{align*}
	\norm{\cov\left(yx\right)}_{op} \le 2\norm{\beta}_2^2 + 1.
	\end{align*}
	Further, if we define $F\left(\widehat{\beta}\right) = \left(\norm{\beta}_2^2+1\right)\mat{I}+\beta\beta^\top$, then it satisfies\begin{align}
	\norm{F\left(\beta\right)-F\left(\widehat{\beta}\right)}_{op} \le 4\norm{\beta}_2\norm{\beta-\widehat{\beta}}_2 + 2\norm{\beta-\widehat{\beta}}_2^2.
	\end{align}
\end{cor}
\begin{proof}
	For the operator norm, of the covariance, using triangle inequality, we have
	\begin{align*}
	\norm{\cov\left(yx\right)}_{op} & = \norm{\left(\norm{\beta}_2^2+1\right)\mat{I} + \beta\beta^\top}_{op} \\
	& \le 2\norm{\beta}_2^2 + 1.
	\end{align*}
	Now for $F$, note we can express it as sum of terms involves difference of $\beta$ and $\widehat{\beta}$.
	\begin{align*}
	& F\left(\beta\right) - F\left(\widehat{\beta}\right) \\
	= & 2\beta^\top\left(\beta-\widehat{\beta}\right)I + \beta\left(\beta-\widehat{\beta}\right)^\top +\left(\beta-\widehat{\beta}\right)\beta^\top- \norm{\beta-\widehat{\beta}}_2^2I - \left(\beta-\widehat{\beta}\right)\left(\beta-\widehat{\beta}\right)^\top.
	\end{align*}
	Therefore, using triangle inequality on the operator norm, we have\begin{align*}
	\norm{F\left(\beta\right)-F\left(\widehat{\beta}\right)}_{op} \le 4\norm{\beta}_2\norm{\beta-\widehat{\beta}}_2 + 2\norm{\beta-\widehat{\beta}}_2^2.
	\end{align*}
\end{proof}
Now to obtain the boundedness assumption, we can use the procedure in Algorithm~\ref{algo:prune_sparse_linear}.
Again, maximal inequality of Gaussian random variables shows with probability $1-\tau$, this procedure does not remove any example sampled from $P$.

%\begin{lem}\label{thm:sparse_linear_no_big_sample}
%	Fix $\tau > 0$, let $Y_1,\cdots,Y_n~\sim\mathcal{N}\left(0,\mat{I}\right)$,with probability $1-\tau$, we have $\norm{Y^i}_2 \le O\left(\sqrt{d\log\left(n/\tau\right)}\right)$ for all $i=1,\cdots,n$.
%\end{lem}

\begin{algorithm}[tb]
	\caption{Pruning for Sparse Linear Regression}
	\label{algo:prune_sparse_linear}
	\begin{algorithmic}[1]
		\STATE \textbf{Input:} $\left\{\left(y_1,x_1\right),\cdots,\left(y_n,x_n\right)\right\}$ 
		\FOR{$i=1,\cdots,n$}
		\IF {$\norm{x_i}_2 = \Omega\left(d\sqrt{\log\left(n/\tau\right)}\right)$ or $\abs{y_i} = \Omega\left(\left(\rho^2+1\right)\sqrt{\log\left(n/\tau\right)}\right)$}
		\STATE remove $\left(y_i,x_i\right)$ from the set.
		\ENDIF
		\ENDFOR
	\end{algorithmic}
\end{algorithm}

It remains to prove the concentration bounds.
When $n = \Omega\left(\frac{s^2\log\left(d/\tau\right)}{\epsilon^2}\right)$, Eqn.~\eqref{eqn:number_bad_points}, Eqn.~\eqref{eqn:spar_mean_hat_tri_good_weights} and Eqn.~\eqref{eqn:spar_mean_square_good_weights} can be proved through classical Bernoulli and Gaussian concentration inequalities.
For the remaining two, the following lemma suffices.
\begin{lem}[Concentration bounds for Sparse Linear Regression]\label{thm:sparse_linear_mean_concentration} 	Suppose for $i=1,\cdots, n$, let \[
x_i\sim N\left(0,\mat{I}\right), y_i = x_i\beta + \xi_i \text{ where } \xi_i \sim N\left(0,1\right).
\] 
Then if $n=\Omega\left(\frac{s\log\left(d/\tau\right)}{\epsilon^2}\right)$, then there is a $\delta = O\left(\epsilon\log^2\left(1/\epsilon\right)\right)$ that with probability at least $1-\tau$, we have for any subset $\set{S} \subset [d]$, $\abs{\set{S}} \le s$ and any $w \in S_{n,\epsilon}$, the followings hold\begin{align*}
	\norm{\sum_{i=1}^{n}w_iy_ix_i^{\set{S}}-\beta^{\set{S}}}_2 & \le \delta \left(\norm{\beta}_2+1\right) \\
	\norm{\sum_{i=1}^{n}w_i\left(y_ix_i^{\set{S}}-\beta^{\set{S}}\right)\left(y_ix_i^{\set{S}}-\beta^{\set{S}}\right)^\top - \left(1+\norm{\beta}_2^2\right)\mat{I}_s - \beta^{\set{S}}\left(\beta^{\set{S}}\right)^\top}_{op} & \le \delta \left(\norm{\beta}_2^2+1\right). 
\end{align*}
\end{lem}
\begin{proof}
We will prove the covariance the concentration for the covariance.
The mean is very similar.
note that \begin{align}
&\sum_{i=1}^{n}w_i\left(y_ix_i^{\set{S}}-\beta^{\set{S}}\right)\left(y_ix_i^{\set{S}}-\beta^{\set{S}}\right)^\top - \left(1+\norm{\beta}_2^2\right)\mat{I}_s - \beta_{\set{S}}\left(\beta_{\set{S}}\right)^\top \nonumber\\
= & \sum_{i=1}^{n}w_ix_i^{\set{S}} \left(x_i\right)^\top \beta\beta^\top x_i \left(x_i^{\set{S}}\right)^\top - \left(\norm{\beta}_2^2\mat{I} + 2 \beta^{\set S}\left(\beta^{\set S}\right)^\top\right) \label{eqn:sparse_linear_fourth}\\
+ & 2\sum_{i=1}^{n}w_i\xi_ix_i^{\set{S}} \beta^\top x_i \left(x_i^{\set{S}}\right)^\top \label{eqn:sparse_linear_third} \\
+ & \sum_{i=1}^{n}w_i\xi_i^2 x_i^{\set{S}}\left(x_i^{\set{S}}\right)^\top - \mat{I} \label{eqn:sparse_linear_second}
\end{align}
We prove the concentration of Eqn.~\eqref{eqn:sparse_linear_fourth}, the Eqn.~\eqref{eqn:sparse_linear_third} and Eqn.~\eqref{eqn:sparse_linear_second} can be proved using similar arguments.
Since $\beta$ is $s$-sparse, it is sufficient to prove that for any $\set{S}' \subset [d]$, $\abs{\set{S}'} \le s$, $\set{S} = \set{S}' \cup \set{S}^*$ where $\set{S}^*$ is the support of $\beta$, the following holds\[
\norm{\sum_{i=1}^{n}w_ix_i^{\set{S}} \left(x_i^{\set {S}}\right)^\top \beta^{\set S}\left(\beta^{\set{S}}\right)^\top x_i^{\set{S}} \left(x_i^{\set{S}}\right)^\top - \left(\norm{\beta}_2^2\mat{I} + 2 \beta^{\set S}\left(\beta^{\set S}\right)^\top\right)}_{op} \le \delta\norm{\beta}_2^2.
\]
Now fix $v \in \mathbb{R}^{\abs{S}}$ with $\norm{v}_2 = 1$.
Define the polynomial $p_v\left(x\right) = v^\top x^{\set{S}}\left(x^{\set{S}}\right)^\top\beta^{\set{S}}$.
By the same argument in the proof of Theorem 4.17 of~\cite{diakonikolas2016robust}, under our assumption on $\delta$ if $n= \Omega\left(\frac{\log\left(1/\tau'\right)}{\epsilon^2}\right)$, for any $w \in S_{n,\epsilon}$, with probability $1-\tau'$\begin{align*}
\abs{\sum_{i=1}^{n}w_ip_v^2\left(x_i\right)-\expect\left[p_v^2\left(x\right)\right]} \le \delta \norm{\beta}_2^2.
\end{align*}
Now take union bound over $\frac{1}{3}$-net of the surface of unit ball of dimension $\abs{S}$ and then take union bound over $\set{S} \in [d]$, we obtain our desired result.
Note when $\norm{\beta}_2^2 \ge 1$, the error in Eqn.~\eqref{eqn:sparse_linear_fourth} will dominate the other two.
On the other hand, if $\norm{\beta}_2^2 \le 1$, Eqn.~\eqref{eqn:sparse_linear_second} will dominate.
Therefore our bound has a $\left(\norm{\beta}_2^2+1\right)$ factor.
\end{proof}

\section{Technical Details of Generalized Linear Models}
\label{sec:sparse_glm}
In this section we consider the generalized linear model (GLM).
Our derivation heavily depends on the following seminal result from Stein.
\begin{thm}[Stein's identity~\citep{stein1971dependent}]
Let $x \sim N\left(0,\mat{I}\right)$ and $G$ a function satisfying some regularity conditions, then \[
\expect\left[G\left(x\right)\cdot x\right] = \expect\left[\triangledown_{x}G\left(x\right)\right].
\]
\end{thm}
We first investigate the basics properties of GLM.
\begin{thm}\label{thm:glm_basics}
If $x \sim N\left(0,\mat{I}\right)$ and $y = u\left(x\beta\right)  + \xi$ where $\xi \sim \mathcal{N}\left(0,\mat{I}\right)$, then we have \begin{align*}
\expect\left[yx\right] &= \expect\left[\bigtriangledown_{x'}u\left(x'\right)\right]\cdot \beta, \\
\expect\left[\left(yx - \expect\left[y x\right]\right)\left(yx-\expect\left[yx\right]\right)^\top\right] &=  \expect\left[\left( 1+u^2\left(x'\right)\right) I + \left(2u\left(x'\right)\bigtriangledown_{x'}^2u\left(x'\right)+\bigtriangledown_{x'}u\left(x'\right)^2\right)\beta \beta^\top\right].
\end{align*} where $x'=x\beta$.
\end{thm}
\begin{proof}
For the first moment, choose $G\left(x\right) = u\left(x\beta\right)$ we directly have the result.
For the covariance, note it is suffice to prove the second moment:\begin{align*}
\expect\left[\left(yx\right)\left(yx\right)^\top\right] &=  \expect\left[\left( 1+u^2\left(x'\right)\right) I + 2\left(u\left(x'\right)\bigtriangledown_{x'}^2u\left(x'\right)+\bigtriangledown_{x'}u\left(x'\right)^2\right)\beta \beta^\top\right].
\end{align*} 
Write $y= u\left(x'\right)+\xi$, since $\expect\left[\xi^2xx^\top\right] = \mat{I}$, we just need to prove \begin{align*}
\expect\left[\left(u\left(x'\right)x\right)\left(u\left(x'\right)x\right)^\top\right] &=  \expect\left[u^2\left(x'\right) I + 2\left(u\left(x'\right)\bigtriangledown_{x'}^2u\left(x'\right)+\bigtriangledown_{x'}u\left(x'\right)^2\right)\beta \beta^\top\right].
\end{align*} 
Choose $G\left(x\right) = u^2\left(x'\right)\cdot x$ in Stein's identity, we have \begin{align*}
\expect\left[\left(u\left(x'\right)x\right)\left(u\left(x'\right)x\right)^\top\right] 
 = \expect\left[g^2\left(x'\right)\right] I + 2\expect\left[u\left(x'\right)\bigtriangledown_{x'}u\left(x'\right)\cdot x\right]\beta^\top.
\end{align*}
Not surprisingly, we can define $G\left(x\right) = u\left(x'\right)\bigtriangledown_{x'}u\left(x'\right)x$ and apply Stein's identity again to obtain the desired result.
\end{proof}
\textbf{Remark:} The expression for linear regression can be derived similarly using Stein's identity.

%In the theorem,
%$\expect\left[\bigtriangledown_{x'}u\left(x'\right)\right]$, $\expect\left[u\left(x'\right)\bigtriangledown_{x'}^2u\left(x'\right)\right]$ and $\expect\left[\bigtriangledown_{x'}u\left(x'\right)\right]$ are just scalar quantities and further notice that $x' \sim \mathcal{N}\left(0,\norm{\beta}_2^2\right)$, so these quantities can be estimated easily just using $y_1,\cdots,y_n$ with one-dimensional robust method like the median estimator.
%Therefore, from now on, we will assume these quantities are known constants.
%Now, if we can estimate $\func = \expect\left[yx\right]$ well, we can recover $\beta$ by scaling.
By this expression, we can define \[F(\beta) = \left( \frac{1 + \mathbb{E} [u^2(x')] }{\left(\mathbb{E} [\nabla_{x'} u(x') ] \right)^2 }  \right) I + \left(\frac{\mathbb{E} [2 u(x') \nabla^2_x u(x') + (\nabla u(x'))^2 ] }{ \left(\mathbb{E} [\nabla_x' u(x') ] \right)^2 } \right) \beta \beta^T. \] 
as the formula for the covariance.
This expression implies that it has the same $L_F$ and $L_{\cov}$ as linear regression up to constant factors.

By maximal inequality of Gaussian and Lipschitz condition of $u$, it is easy to show Algorithm~\ref{algo:prune_glm} will not remove any sample from $P$ with probability at least $1-\tau$.
\begin{algorithm}[tb]
	\caption{Pruning for Generalized Linear Models}
	\label{algo:prune_glm}
	\begin{algorithmic}[1]
		\STATE \textbf{Input:} $\left\{\left(y_1,x_1\right),\cdots,\left(y_n,x_n\right)\right\}$ 
		\FOR{$i=1,\cdots,n$}
		\IF {$\norm{x_i}_2 = \Omega\left(d\sqrt{\log\left(n/\tau\right)}\right)$ or $\abs{y_i} = \Omega\left(u\left(0\right)+\left(\rho^2+1\right)\sqrt{\log\left(n/\tau\right)}\right)$}
		\STATE Remove $\left(y_i,x_i\right)$ from the set.
		\ENDIF
		\ENDFOR
	\end{algorithmic}
\end{algorithm}
now it remains to prove concentrations for $yx$ and $y^2xx^\top$.
When $n = \Omega\left(\frac{s^2\log\left(d/\tau\right)}{\epsilon^2}\right)$ Eqn.~\eqref{eqn:number_bad_points}, Eqn.~\eqref{eqn:spar_mean_hat_tri_good_weights} and Eqn.~\eqref{eqn:spar_mean_square_good_weights} can be proved through classical Bernoulli concentration inequality and Lipschitz function of Gaussian variable concentration inequality.
Now we prove the remaining two concentration inequalities.
The technique we used is very similar to the proof of Theorem 4.17 of~\citep{diakonikolas2016robust}.
\begin{lem}
Suppose for $i=1,\cdots,n$,  \begin{align*}
x_i \sim N\left(0,\mat{I}\right), y_i = u\left(x_i\beta\right) + \xi_i \text{ where } \xi_i \sim N\left(0,1\right).
\end{align*} 
with $\norm{\beta}_0 \le s$ and $u$ is a known link function with $u(0) = O\left(1\right)$ and $1$-Lipschitz.
If $n = \Omega\left(\frac{s\log\left(d/\tau\right)}{\epsilon^2}\right)$, then there is a $\delta = O\left(\epsilon\log^2\left(\frac{1}{\epsilon}\right)\right)$ that with probability at least $1-\tau$ we have for any subset $\set{S} \subset [d]$, $\abs{\set S}\le s$ and for any $w \in S_{n,\epsilon}$, we have \begin{align*}
\norm{\sum_{i=1}^{n}w_iy_ix_i^{\set{S}}-\expect\left[yx^{\set S}\right]}_2 &= \delta\left(\norm{\beta}_2+1\right), \\
\norm{\sum_{i=1}^{n}w_i\left(y_ix_i^{\set{S}}-\expect\left[yx^{\set S}\right]\right)\left(y_ix_i^{\set{S}}-\expect\left[yx^{\set S}\right]\right)^\top -\expect\left[\left(yx^{\set S} - \expect\left[y x^{\set S}\right]\right)\left(yx^{\set S}-\expect\left[yx^{\set S}\right]\right)^\top\right]}_{op} & \le   \delta\left(\norm{\beta}_2^2+1\right)
\end{align*} where $x'=x\beta$.
\end{lem}
\begin{proof}
We will prove the covariance concentration because the mean concentration is quite similar.
Similar to Theorem~\ref{thm:sparse_linear}, we can divide the expression into 5 parts\begin{align}
&\sum_{i=1}^{n}w_i\left(y_ix_i^{\set{S}}-\expect\left[yx^{\set S}\right]\right)\left(y_ix^{\set{S}}-\expect\left[yx^{\set S}\right]\right)^\top -\expect\left[\left(yx^{\set S} - \expect\left[y x^{\set S}\right]\right)\left(yx^{\set S}-\expect\left[yx^{\set S}\right]\right)^\top\right] \nonumber\\
= &\sum_{i=1}^{n}w_i\left(u\left(x_i\beta\right)-u\left(0\right)\right)^2x^{\set S}\left(x_i^{\set{S}}\right)^\top - \expect\left[\left(u\left(x\beta\right)-u\left(0\right)\right)^2x^{\set S}\left(x^{\set S}\right)^\top\right] \label{eqn:glm_cov_concen}\\
+ & 2\sum_{i=1}^{n}w_iu\left(0\right)u\left(x_i\beta\right)x_i^{\set{S}}\left(x_i^{\set{S}}\right)^\top - 2u\left(0\right)\expect\left[u\left(x\beta\right)x^{\set S}\left(x^{\set S}\right)^\top\right] \nonumber\\
+ &  \sum_{i=1}^{n} w_iu^2\left(0\right)x^{\set S}\left(x_i^{\set{S}}\right)^\top - u^2\left(0\right)\mat{I} \nonumber\\
+ & 2\sum_{i=1}^{n}w_i\xi_iu\left(x\beta\right)x_i^{\set{S}}\left(x_i^{\set{S}}\right)^\top \nonumber\\
+ & \sum_{i=1}^{n} w_i\xi_i^2 x^{\set S}\left(x^{\set S}\right)^\top- \mat{I} \nonumber.
\end{align}
Again we will prove the concentration for Eqn.~\eqref{eqn:glm_cov_concen}, the remaining terms can be bounded similarly.
Now fix $\set S' \subset [d]$ and let  $\set{S} = \set{S}' \cup S^\star$ where $S^\star$ is the support of $\beta$.
For a fixed $v \in \mathbb{R}^{\abs{S}}, \norm{v}_2 = 1$, define \begin{align}
\label{eqn:pv}
p_v\left(x\right) = \left(u\left(x^{\set S}\beta\right)-u\left(0\right)\right)\left(v^\top x^{\set S}\right).
\end{align}
For some fixed large enough constant $c$, by basic Gaussian concentration inequalities we have \[
\prob\left[\left(v^\top x^{\set S}\right)^2 \ge  \sqrt{c}\log\left(\frac{1}{\epsilon}\right)\right] = O\left(\epsilon\right).
\]
Similarly, using Lipschitz condition we have\begin{align*}
&\prob \left[\left(u\left(x^{\set S}\beta\right) -u\left(0\right)\right)^2\ge \sqrt{c}\log\frac{1}{\epsilon}\right] \\
\le &\prob \left[
\abs{x\beta}^2 \ge \sqrt{c}\log\frac{1}{\epsilon}\norm{\beta}_2^2 
\right]\\
= & O\left(\epsilon\right)
\end{align*}
Therefore, we have \begin{align*}
\prob\left[p_v^2\left(x\right) \ge c\log^2\frac{1}{\epsilon}\norm{\beta}_2^2\right] = O\left(\epsilon\right).
\end{align*}
Now applying Hoeffding inequality we have if $n = \Omega\left(\frac{\log\left(1/\tau\right)}{\epsilon^2}\right)$, with probability $1-\tau$:
\begin{align*}
\frac{1}{n}\abs{\left\{i:p_v^2\left(x_i\right) \ge c\log^2\frac{1}{\epsilon}\norm{\beta}_2^2\right\}} &\le 2\epsilon.
\end{align*}
Now define a distribution $\mathcal{D}$ that for $A_i \sim \mathcal{D}$, $A_i = p_v^2\left(x_i\right)$ if $p^2\left(x_i\right) \le c\norm{\beta}_2^2 \log^2\frac{1}{\epsilon}$ and $0$ otherwise.
Let $\alpha'$ be the expectation of mean of $D$. 
By Hoeffding inequality we can show if $n \ge \Omega\left(\frac{\log\left(1/\tau\right)}{\epsilon^2}\right)$ we have with probability $1-\tau$, \[
\abs{\frac{1}{n}\sum_{i=1}^{n}A_i - \alpha'} = O\left(\epsilon\right).
\]
Now let $\alpha = \expect\left[p_v^2\left(x\right)\right] = O\left(\norm{\beta}_2^2\right)$.
We have \begin{align*}
\abs{\alpha'-\alpha} &= \expect_{x \sim N\left(0,I\right)}\left[p_v^2\left(x\right) \ind_{p_v^2\left(x\right) \ge \left(c\log^2\frac{1}{\epsilon}\right)}\right]\\
&= \int_{\sqrt{c}\log\frac{1}{\epsilon}}^{\infty}t^2\prob\left[p_v^2\left(x\right) \ge t^2\right]dt\\
& = \int_{\sqrt{c}\log\frac{1}{\epsilon}}^{\infty}t^2\prob\left[\left(x^{\set S}\beta\right)^2\left(x^{\set S}v\right)^2 \ge t^2\right]dt\\
& = O\left(c\log^2\frac{1}{\epsilon}\norm{\beta}_2^2\right).
\end{align*}
Therefore we have if $n = \Omega\left(\frac{\log\frac{1}{\tau}}{\epsilon^2}\right)$, with probability $1-\tau$\begin{align*}
\abs{\frac{1}{n}\sum_{i=1}^{n}A_i - \alpha} &= O\left(\epsilon\log^2\frac{1}{\epsilon}\right).
\end{align*}
Now condition on the followings :\begin{align*}
\frac{1}{n}\abs{\left\{i:p_v^2\left(x_i\right) \ge c\log^2\frac{1}{\epsilon}\norm{\beta}_2^2\right\}} &\le 2\epsilon\\
\abs{\frac{1}{n}\sum_{i=1}^{n}A_i - \alpha} &= O\left(\epsilon\log^2\frac{1}{\epsilon}\right).
\end{align*}
Define $J^\star$ the largest $2\epsilon n$ indices of $p_v^2\left(x_i\right)$s and $J_1^\star = \left\{i: p_v^2\left(x_i\right) \ge c\log^2\frac{1}{\epsilon}\norm{\beta}_2^2\right\}$.
By the conditions, we known $J_1^\star \subset J^\star$ and \[
\abs{\frac{1}{n}\sum_{i \notin J_1^{\star}}p_v^2\left(x_i\right)-\alpha} = O\left(\epsilon\log^2\frac{1}{\epsilon}\norm{\beta}_2^2\right).
\]
For any index set $I$ with $\abs{I}=\left(1-2\epsilon\right)n$, divide $[n] \setminus I = J^+ \cup J^-$ where $J^+ = \left\{i \in I: p_v^2\left(x_i\right)\ge \alpha \right\}$ and $J^- = \left\{i \in I: p_v^2\left(x_i\right) < \alpha\right\}$.
First we prove the upper bound \begin{align*}
&\frac{1}{\left(1-2\epsilon\right)n}\sum_{i \in I}w_i\left(p^2\left(x_i\right)-\alpha\right) \\
\le & \frac{1}{\left(1-2\epsilon\right)n}\sum_{i\in I \cup J^+}\left(p_v^2\left(x_i\right)-\alpha\right) - \frac{1}{\left(1-2\epsilon\right)n}\sum_{i\in J^-}\left(p_v^2\left(x_i\right)-\alpha\right)\\
\le & \frac{1}{\left(1-2\epsilon\right)n}\abs{\sum_{i=1}^{n}p_v^2\left(x_i\right)-\alpha} + \frac{2}{\left(1-2\epsilon\right)n}\abs{\sum_{i \in J^-}\left(p_v^2\left(x_i\right)-\alpha\right)} \\
= & O\left(\epsilon\norm{\beta}_2^2\right) + \frac{\abs{J^-}}{\left(1-2\epsilon\right)n}\alpha\\
= & O\left(\epsilon\norm{\beta}_2^2\right)
\end{align*}
where in the fourth line we used concentration inequality of Lipschitz function of Gaussian random variables.
For the lower bound\begin{align*}
\frac{1}{\left(1-2\epsilon\right)n}\sum_{i\in I}\left(p_v^2\left(x_i\right)-\alpha\right) &\ge 
\frac{1}{\left(1-2\epsilon\right)n}\sum_{i\notin J_1^\star}\left(p_v^2\left(x_i\right)-\alpha\right)\\
& \ge - O\left(\epsilon\log^2\frac{1}{\epsilon}\norm{\beta}_2^2\right).
\end{align*}
Note this holds for any $I$ and by convexity for any $w \in S_{n,\epsilon}$ we can conclude that Eqn.~\eqref{eqn:glm_cov_concen} holds for fixed $S$ and $v$.
Now take union bounds over $\frac{1}{3}$-net of the surface of unit ball of dimension $\abs{S}$ and subsets of $[d]$ with cardinality $2s$, we obtain the desired result.

Similar to sparse linear regression, the final bound depends on whether $\norm{\beta}_2^2$ is larger than $1$ or not, which leads to the form of our bound.
\end{proof}

\section{Technical Details of Logistic-type Models}
\label{sec:sparse_classification}
In this section we consider the generalized linear model for binomial label.
\begin{thm}
Suppose $x \sim N\left(0,\mat{I}\right)$ and $y= u\left(x\beta\right)+\xi\left(x\beta\right)$ where $g$ is a known link function and \begin{align*}
\xi\left(x\beta\right) = \begin{cases}
-u\left(x\beta\right) &\text{ w.p } \quad 1-u\left(x\beta\right) \\
1-u\left(x\beta\right) &\text{ w.p } \quad u\left(x\beta\right)
\end{cases}
\end{align*}
then we have \begin{align*}
\expect\left[yx\right]  & = \expect\left[\bigtriangledown_{x'}u\left(x'\right)\right]\cdot \beta \\
\expect\left[\left(yx-\expect\left[yx\right]\right) \left(yx-\expect\left[yx\right]\right)^\top\right] & =  \expect\left[u\left(x'\right)\right]\mat{I} +  \left(\expect\left[\bigtriangledown_{x'}^2u\left(x'\right)\right]-\expect\left[\bigtriangledown_{x'}u\left(x'\right)\right]^2\right)\cdot \beta\beta^\top 
\end{align*} where $x' = x\beta$.
\end{thm}
\begin{proof}
This is a simple application of Stein's identity.
The derivation is similar to sparse generalized linear model.
\end{proof}
%Similar to Sec.~\ref{sec:sparse_glm}, since $\expect\left[u\left(x'\right)\right]$, $\expect\left[\bigtriangledown_{x'}u\left(x'\right)\right]$ and $\expect\left[\bigtriangledown_{x'}^2u\left(x'\right)\right]$ are scalars we can estimate them by one-dim method using $y_1,\cdots,y_n$.
To achieve the boundedness condition, we resort to Algorithm~\ref{algo:prune_glm_bin}.
Finally, the concentration bounds can be proved using the exactly same arguments in Sec.~\ref{sec:sparse_glm}.
Notice that the function $p_v$ defined in Eqn.~\eqref{eqn:pv} has a better concentration property:\[
\prob\left[p_v^2\left(x\right) \ge c\log\left(\frac{1}{\epsilon}\right)\right] = O\left(\epsilon\right)
\]
because of the boundedness of $u\left(\cdot\right)$.
This fact leads to a slightly stronger bound than that of generalized linear models.
\begin{algorithm}[tb]
	\caption{Pruning for Logistic-type Models}
	\label{algo:prune_glm_bin}
	\begin{algorithmic}[1]
		\STATE \textbf{Input:} $\left\{\left(y_1,x_1\right),\cdots,\left(y_n,x_n\right)\right\}$ 
		\FOR{$i=1,\cdots,n$}
		\IF {$\norm{y_ix_i}_2 = \Omega\left(\abs{\expect\left[\bigtriangledown_{x'}u\left(x'\right)\right]}\sqrt{d\log\left(n/\tau\right)}\right)$.}
		\STATE Remove $\left(y_i,x_i\right)$ from the set.
		\ENDIF
		\ENDFOR
	\end{algorithmic}
\end{algorithm}

\end{document}